\documentclass{article}

\usepackage{microtype}
\usepackage{graphicx}
\usepackage{booktabs} %

\usepackage{hyperref}

\usepackage[arXiv]{icml2023}

\usepackage{amsmath}
\usepackage{amssymb}
\usepackage{mathtools}
\usepackage{amsthm}

\theoremstyle{plain}
\newtheorem{theorem}{Theorem}[section]

\newtheorem{lemma}[theorem]{Lemma}

\theoremstyle{definition}

\theoremstyle{remark}

\usepackage{amsmath,amsfonts,bm}

\def\eqref#1{equation~\ref{#1}}

\def\1{\bm{1}}

\def\ru{{\textnormal{u}}}
\def\rv{{\textnormal{v}}}
\def\rw{{\textnormal{w}}}

\def\rvu{{\mathbf{i}}}

\def\rvu{{\mathbf{u}}}

\def\rvw{{\mathbf{w}}}

\def\vtheta{{\bm{\theta}}}

\def\vlambda{{\bm{\lambda}}}
\def\vbeta{{\bm{\beta}}}

\def\vh{{\bm{h}}}

\def\vx{{\bm{x}}}

\DeclareMathAlphabet{\mathsfit}{\encodingdefault}{\sfdefault}{m}{sl}
\SetMathAlphabet{\mathsfit}{bold}{\encodingdefault}{\sfdefault}{bx}{n}

\def\gB{{\mathcal{B}}}
\def\gC{{\mathcal{C}}}

\def\gE{{\mathcal{E}}}

\def\gG{{\mathcal{G}}}

\def\gV{{\mathcal{V}}}

\def\gHG{{\mathcal{HG}}}

\def\bp{{\gB}}
\def\pg{{\gC}}

\newcommand{\ZP}{\mathbb{Z}_+}

\newcommand{\softmax}{\mathrm{softmax}}

\newcommand{\parents}{Pa} %

\usepackage{url}

\usepackage{nicefrac}       %
\usepackage{adjustbox}
\usepackage{makecell}       %
\usepackage{arydshln}       %

\usepackage{graphicx}
\usepackage{booktabs} %
\usepackage{appendix} %

\input{Definitions.tex} %

\usepackage{sidecap}
\usepackage{placeins}
\usepackage{rotating}
\usepackage{caption}
\usepackage{float}
\usepackage[caption = false]{subfig}
\usepackage{tikz}
\usetikzlibrary{fit,positioning}
\usetikzlibrary{arrows,automata}

\renewcommand{\captionsize}{\footnotesize}

\icmltitlerunning{On Hierarchical Multi-Resolution Graph Generative Models}

\begin{document}

\twocolumn[
\icmltitle{On Hierarchical Multi-Resolution Graph Generative Models}

\icmlsetsymbol{equal}{*}

\begin{icmlauthorlist}
\icmlauthor{Mahdi Karami}{MHD}
\icmlauthor{Jun Luo}{Ark}
\end{icmlauthorlist}

\icmlaffiliation{MHD}{Correspondence to: Mahdi Karami $<$mahdi.karami@ualberta.ca$>$.}
\icmlaffiliation{Ark}{Noah's Ark Lab, Toronto, Canada}

\icmlcorrespondingauthor{}{mahdi.karami@ualberta.ca}

\icmlkeywords{Machine Learning, Graph Neural Network, Generative Models}

\vskip 0.3in
]

\printAffiliationsAndNotice{}  %

\begin{abstract}
In real world domains, most graphs naturally exhibit a  hierarchical structure.
However, data-driven graph generation is yet to effectively capture such structures.
To address this, we propose a novel approach that recursively generates community structures at multiple resolutions, with the generated structures conforming to training data distribution at each level of the hierarchy.
The graphs generation is designed as a sequence of coarse-to-fine generative models allowing
for parallel generation of all sub-structures, resulting in a high degree of scalability.
 Our method demonstrates generative performance improvement on multiple graph datasets.
\end{abstract}

\section{Introduction} \label{sec:Intro}

Graphs are ubiquitously relevant and useful for representing relations.
Data-driven approaches to graph generation is both challenging and and highly valuable for various applications \citep{dai2020scalableGraphGen}.
These include document generation \citep{blei2003LDA}, discovering new molecular and chemical structures,
generation and analysis of realistic data networks, knowledge graph generation for recommendation systems \citep{he2021pgg},
synthesizing scene graphs in computer vision and virtual reality \citep{habitat19iccv, ramakrishnan2021hm3d},
and generation of interactive scenarios for autonomous driving simulation platforms such as CARLA \citep{dosovitskiy17CARLA} and SMARTS \citep{zhou2020smarts}.

There are natural hierarchical community structures in all the domains mentioned above, such as paragraphs, sentences, and words of a document, communities in user-item graphs, room of an apartment on a floor, columns of cars and groups of pedestrians on a city block.
On the one hand, higher level relations, such as how close two groups of pedestrians are, reflect high-level interactions between communities.
On the other hand, low-level relations and distributions, such as how dense a group of pedestrians is, reflect the local  structures.
Realistic graph generation models must learn both of these interactions and be able to capture the cross-level relations.
While hierarchical, multi-resolution generative models were developed for specific data types such as voice \cite{oord2016wavenet}, image \cite{reed2017parallel-multiscale-DE, karami2019invertible} and molecular motifs \cite{jin2020hierarchicalMolec}, these methods rely on domain-specific priors that are not be suitable  for general graphs.
To the best of our knowledge, there exists no generation models suitable for generic graphs that are both learned from data and handle interacting semantic hierarchies.

Graph generative models has been studied extensively.
Classical methods from \cite{erdos1960evolution} and \cite{barabasi1999emergence} are based on random graph theory that can only capture a set of hand engineered graph statistics.
\citet{leskovec2010kronecker} proposed a scalable generative model based on the Kronecker product of matrices that can learn some graph properties such as degree distribution, but is very limited in modeling the underlying distributions.
These models fail to capture important graph properties such as community structure in large family of graphs.
Motivated by recent advances in recurrent neural networks (RNN) and graph neural networks (GNN),
various neural network based generative models has been proposed \cite{you2018graphrnn, li2018learning, liao2019GRAN}.
These methods, which belong to the family of autoregressive algorithms, generate graphs as a sequence of edges or nodes so they highly rely on a appropriate node ordering and they don't take into account the community structures present within graphs.
Moreover, due to their recursive nature they are  computationally expensive for moderately large graphs.

Herein, we propose an efficient hierarchical Multi-Resolution Generative (MRG) model
to address the limitations of existing generative models
by capturing community structures and cross-level interactions.
The proposed model captures the hierarchical relations by allowing the representation of a node at each level to depend not only on its community but also on its corresponding super-node at the higher level.
This approach allows the generation process at the lower level to be independent of the specific ordering of the nodes at the higher levels, reducing the overall sensitivity to initial random permutations.
The graphs generation is designed as a sequence of  coarse-to-fine generative models where given a higher level (lower resolution) graph, the generation of the communities in the lower level can be performed in parallel, resulting in a high degree of scalability through parallelism.
The output distribution of edges are parameterized by a multinomial distribution and a recursive factorization is derived for this distribution which enable the use of existing autoregressive methods for generating communities.
This results an expressive distribution that can additionally models the graphs with integer-valued edge weights.

\begin{figure*}[t!]
\centering
	\subfloat[\label{fig:HG}]{\includegraphics[width=0.3\linewidth]{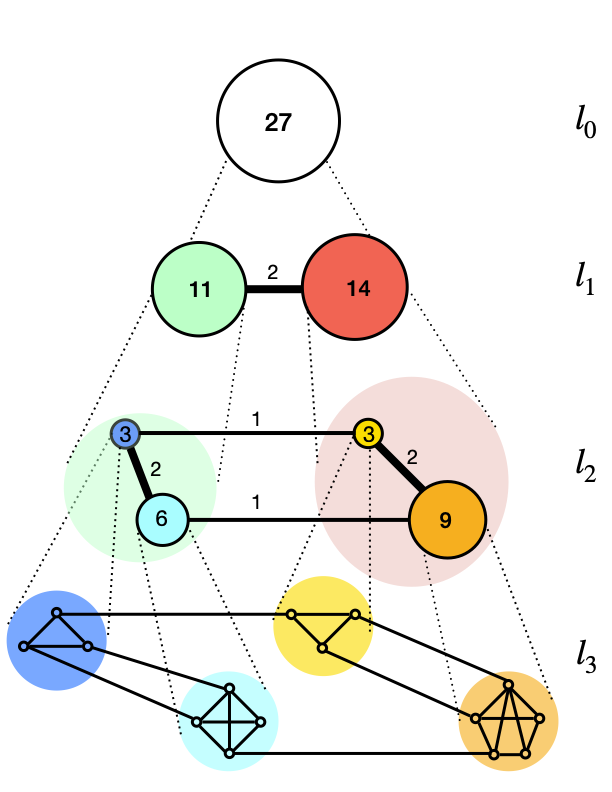}} \hfill
 	\subfloat[\label{fig:ADJ}]{\includegraphics[width=0.3\linewidth]{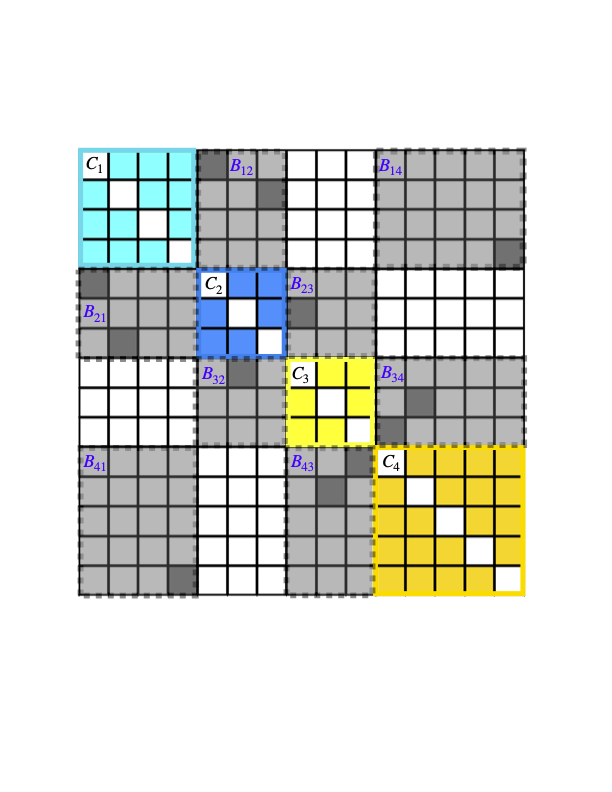}} \hfill
   	\subfloat[\label{fig:MN_BN}]{\includegraphics[width=0.27\linewidth]{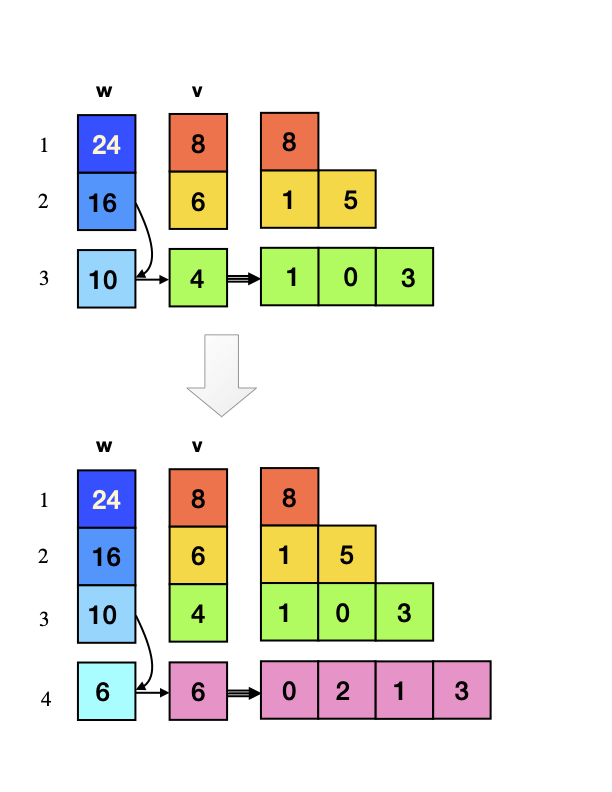}} %
        \caption{
        \captionsize
        (a) A sample hierarchical graph with 3 levels is shown. Communities are shown in different colors and the weight of a node and the weight of an edge in a higher level, represent the sum of the in-community  connections in the corresponding community and the total weight of corresponding bipartite, respectively. Node size and edge width indicate their weights.
        (b)	The matrix shows corresponding adjacency  of the graph $\gG^2$ matrix where each of its sub-graphs corresponds to a block in the adjacency matrix, partition graphs are in shown different colors and bipartites are colored in gray.
        (c) Decomposition of multinomial distribution as a recursive \textit{stick-breaking} process where at each iteration, first a fraction of the remaining weights $\rw_m$ is allocated to the $m$-th row (the $m$-th node in the sub-graph) and then this fraction $\rv_m$ is distributed among that row of lower triangular adjacency matrix, $\hat{A}$.
	}
	\vskip -.1in
\end{figure*}

\section{Problem Formulation}

A graph $ \gG = \left(\gV,~ \gE \right) $ is defined as a set of nodes (vertices) $\gV$ and edges $\gE$ with sizes $n = |\gV|$ and $m=|\gE|$ and adjacency matrix $\Amat^{\pi}$ for the node ordering $\pi$.
A graph can be decomposed to
\textit{partition graphs} (a.k.a. community or cluster), denoted by $ \pg_i = \left(\gV(\pg_i),~ \gE(\pg_i)\right) $
with adjacency matrix $\Amat_{i}$,
and %
\textit{bipartite graphs}, denoted by  $ \bp_{ij} = \left(\gV(\pg_i),~\gV(\pg_j),~ \gE(\bp_{ij})\right) $.
A bipartite is composed of the cross links of two neighboring partition graphs and its adjacency matrix is denoted by $\Amat_{ij}$.
Each partition graph can be aggregated to a super-node and for each bipartite a super-edge is added to link neighboring communities. This results in a coarsened graph at the higher level.
Herein, the levels are indexed by superscripts.
Formally, each partition graph at level $l$, $\pg_{i}^l$, is mapped to a node the higher level graph,
also called its parent node, $v^{l-1}_{i} = \parents(\pg_{i}^l) $ and
each bipartite at level $l$ %
is represented by an edge in the higher level, also called its parent edge, $ e^{l-1}_{i} = \parents(\bp_{ij}^l) = <v^{l-1}_{i}, v^{l-1}_{j}> $.
The weights of the self edges of these parent nodes and edges are determined by the sum of the weights of the edges within the partition graph and bipartite, \ie $ w^{l-1}_{ii} = \sum_{e\in \gE(\pg_i^l)} w_e$ and $w^{l-1}_{ij} = \sum_{e\in \gE(\bp_{ij}^l)} w_e$,  respectively.
This is how an coarsened graph $\gG^{l-1}$ with integer-valued weighted edges is created in the higher level.

This process continues recursively in a bottom-up manner until a single node graph $\gG^0$ is obtained
which results in a hierarchical graph (hyper-graph) $\gHG$, defined by the set of graphs in all levels of abstractions
, $\gHG := \{ \gG^0,  ...., \gG^{L-1}, \gG^L \} $,
where leaf level $\gG^L$ is the final graph that is being generated and $\gG^0$is the root graph. %
An $\gHG$ is visualized in figure \ref{fig:HG}.
This hierarchical tree structure allows modeling of both short and long-range interactions among nodes, as well as controlling the flow of information between them, across multiple levels of abstraction
which is a key aspect of our proposed generative model.

\paragraph{Community detection}
Different community detection algorithms have been proposed that try to identify communities based on specific metrics.
or cluster nodes with similar features.
The Louvain algorithm \citep{blondel2008Louvain} is a popular method for graph coarsening, which is a process of reducing the resolution of a graph by grouping similar nodes together. It is a community detection algorithm that iteratively detects communities by maximizing a modularity function. The algorithm starts with each node as its own community and then repeatedly merges communities based on the highest increase in modularity until no further improvement can be made. The resulting communities form a coarser graph with fewer nodes, where each node represents a group of nodes from the original graph.
This heuristic algorithm is computationally efficient and scalable to large graphs for community detection based on graph topology, making it a suitable choice for our graph coarsening step.

\section{Hierarchical Multi-Resolution Graph Generation} \label{sec:HG_gen}
This work aims to establish a multi-resolution framework that generates graph
in a coarse-to-fine approach.
Given a higher level graph, the graph at its subsequent (child) level can be specified by a conditional probability,
and this process can be repeated until the lowest level, or leaf level, is attained.
This concept is formally captured in the following theorem.

\begin{theorem} \label{thm:multi_level}
Given a graph $\gG$  and an ordering $\pi$, assuming there is a deterministic function that provides the corresponding high-level graphs in a hierarchical order as $\{ \gG^L, \gG^{L-1}, ... , \gG^0 \}$, then:

\begin{align}
    & p(\gG = \gG^L , \pi) = p( \{ \gG^L, \gG^{L-1}, ... , \gG^0 \} , \pi)  \nonumber\\
    &= p(\gG^L , \pi ~|~ \{\gG^{L-1}, ... , \gG^0 \})~
     ... ~
     p(\gG^{1} , \pi ~|~ \gG^0) ~ p(\gG^{0}) \nonumber\\
    &= \prod_{l=0}^{L} p(\gG^l , \pi ~|~ \gG^{l-1}) \times p(\gG^{0})
\end{align}
\end{theorem}
\begin{proof}
Last equality holds as the graphs at the coarser levels are derived from the finer level graphs.
\end{proof}

\subsection{Community-based Graph Generation} \label{sec:graph_gen}
Based on the community structure of a hyper-graph, conditional probability of a graph at level $l$, $p(\gG^l ~|~ \gG^{l-1})$, can be decomposed according to its partition graphs and bipartite graphs:
\begin{align} \label{eq:decomp}
    & p(\gG^l ~|~ \gG^{l-1}) \nonumber \\
    & = p(\{\pg_{i}^l ~~\forall i \in \gV_{\gG^{l-1}}\} \cup
    \{\bp_{ij}^l ~~\forall <i,j> \in \gE_{\gG^{l-1}}  \} ~|~ \gG^{l-1}) \nonumber \\
    & \approxeq \prod_{i ~ \in ~ \gV_{\gG^{l-1}}} p(\pg_{i}^l  ~|~ \gG^{l-1})
    \prod_{\underset{\in ~ \gE_{\gG^{l-1}}}{<i,j>}  } p(\bp_{ij}^l  ~|~ \gG^{l-1}, \pg_{i}^l, \pg_{j}^l)
\end{align}

Accordingly, the log-likelihood of $\gG^l$ can be decomposed as the log-likelihood of its sub-structures:
\begin{align} \label{eq:logliklehood_decomp}
\log p_{\phi^l}(\gG^l ~|~ \gG^{l-1}) =
\sum_{i \in \gV_{\gG^{l-1}}} \log p_{\phi^l} (\pg_{i}^l  ~|~ \gG^{l-1}) \nonumber\\
+
\sum_{<i,j> \in \gE_{\gG^{l-1}}} \log p_{\phi^l} (\bp_{ij}^l  ~|~ \gG^{l-1}, \pg_{i}^l, \pg_{j}^l)
\end{align}

Here, we assume that
the partition graph $\pg_{i}^l$ is independent of all other components in its level given the parent graph $\gG^{l-1}$.
Additionally, the bipartite graphs $\bp_{ij}^l$ are assumed to be independent of the rest of components given the parent graph $\gG^{l-1}$ and their corresponding pairs of parts $(\pg_{i}^l, \pg_{j}^l)$.\footnote{
Indeed, this assumption implies that the cross dependency between partition graphs are primarily encoded by their  parent abstract graph which is reasonable where the nodes' dependencies are mostly local and are within community rather than being global.}
Therefore, given the graph at a higher level, the generation of graph at its following level is reduced to generation of its partition and bipartite sub-graphs.
As illustrated in figure \ref{fig:ADJ}, each of these sub-graphs corresponds to a block in the adjacency matrix, so the proposed hierarchical model generates adjacency matrix in a blocks-wise fashion and constructs the final graph topology.
As a result, the generation of the partitions in each level can be performed in parallel
and subsequently, the generation decisions of all bipartites in each level may occur at one pass.

For each partition graph, the Graph Recurrent Attention Network (GRAN) \cite{liao2019GRAN} is adopted
where generation of each partition graph is performed sequentially in a node-by-node manner
and the generative probability is factorized by the probability of rows of the lower triangle adjacency matrix block $\hat{\Amat}_{i}^l$.
The generation decisions in this model are formulated as a function of the node representations obtained by an attention based graph neural network (GNN).
In short, it first computes the initial feature of the nodes as a linear mapping of their corresponding row in the lower triangle adjacency matrix and then learns the node representation $h_i$ using \textit{GNN with attentive messages} model.
We simply denote it as $h_i = \mathrm{GNN}^l(\gG_{in}; {\gamma^l} )$ that is a GNN parameterized by a set of parameters $\gamma^l$. Please refer to \cite{liao2019GRAN} for detailed formulation.
In the following, we modify the final generative probabilities, $p(\pg_{i}^l  ~|~ \gG^{l-1})$, to include the state of the parent graph and also to model the non-negative integer valued weights of the edges.

\subsection{Probability Distribution of Candidate Edges\label{sec:dist}}
In a hierarchical graph, the edges has non-negative integer valued weights while the sum of all the edges in partition graph $ \pg_i^l $ and bipartite graph $ \bp_{ij}^l $ are determined by their corresponding edges in the parent graph, \ie $w^{l-1}_{ii}$ and $w^{l-1}_{ij}$ respectively.
Therefore, the edge weights in each subgraph can be modeled as a multinomial distribution.
So, let's denote the set of all candidate edges of the bipartite $ \bp_{ij}^l $ by a random vector
$\rvw := [w_e]_{e ~\in~ \gE({\bp_{ij}^l})}$ , its probability can be described as
\begin{align} \label{eq:mn_bp}
\rvw & \sim \text{Mu}(\rvw~|~w^{l-1}_{ij}, \vtheta^{l}_{ij} )  \nonumber \\
& = \frac{w^{l-1}_{ij} !}{\prod_{e=1}^{|\gE(\bp_{ij}^l)|} \rvw[e] ! } \prod_{e=1}^{|\gE(\bp_{ij}^l)|} {(\vtheta^{l}_{ij} [e])}^{\rvw[e] }
\end{align}
where
$\{\vtheta^{l}_{ij}[e] ~|~ \vtheta^{l}_{ij}[e]\ge 0, ~ \sum \vtheta^{l}_{ij}[e] = 1 \}$ are the parameter of the distribution,
and
the multinomial coefficient $\frac{n!}{\prod \rvw[e]  ! }$ is the number of ways to distribute the total weight $w^{l-1}_{ij} = \sum_{e=1}^{|\gE(\bp_{ij}^l)|}\rvw[e]  $ into all candidate edges of $\bp_{ij}^l$.\footnote{
It is analogous to the random trial of putting $n$ balls into $k$ boxes, where the joint probability of the number of balls in all the boxes follows the multinomial distribution.
}

Likewise, the probability distribution of the set of candidate edges for each partition graph can
be modeled by a multinomial distribution
but since the generation decisions happens as a sequential process, we are interested to decomposed this probability distribution accordingly.

\begin{lemma} \label{thm:mn2bn}
A random vector $\rvw \in \ZP^E$ with multinomial distribution %
can be recursively decomposed to a sequence of binomial distributions:
\begin{align}
    \text{Mu}(\rw_1, ~ ..., \rw_E ~ &| ~ w,~ [\theta_1, ~ ..., \theta_E ]) \nonumber \\
    &= \prod_{e=1}^{E} \text{Bi}(\rw_{e} ~ | ~ w - \sum\nolimits_{i<e}\rw_i, \hat{\theta}_e), \\
    \text{where: } \hat{\theta}_e &= \frac{\theta_{e}}{1-\sum_{i<e}\theta_i} \nonumber
\end{align}
This decomposition is a \textit{stick-breaking} process where $\hat{\theta}_e$ is the fraction of the remaining probabilities we take away every time and allocate to the $e$-th component \citep{linderman2015dependentMultinomial}.
\end{lemma}
This lemma offers modeling the generation of a partition graph as an edge-by-edge recursive generation process hence is analogous to autoregressive algorithms such as GraphRNN \citep{you2018graphrnn} with $\mathcal{O}(|\gV_{\pg}|^2)$ generation steps.
As a more efficient alternative, we are interested in generating a partition graph one node at a time which entails decomposing the edges probability in a group-wise form where the candidate edges between the $t$-th node and the already generated graph are grouped together.
In the following theorem we formally derive such decomposition for multinomial distributions.

\begin{theorem} \label{thm:mn2bnmn}
For a random counting vector $\rvw \in \ZP^E$ with multinomial distribution $\text{Mu}(\rvw~ | ~ w, \vtheta)$, %
let's split it into $M$ disjoint groups $\rvw=[\rvu_1, ~ ..., \rvu_M ]$ where $\rvu_m \in \ZP^{E_m} ~,~ \sum_{m=1}^M {E_m} = E $,
and also split the probability vector as $\vtheta=[\vtheta_1, ~ ..., \vtheta_M ]$.
Additionally, let's define sum of all variables in the $m$-th group by a random count variable $\rv_m := \sum_{e=1}^{E_m} \ru_{m,e}$.
Then the multinomial distribution can be modeled as a chain of binomials and multinomials:
\begin{align} \label{eq:mn2bnmn}
    &\text{Mu}(\rvw=[\rvu_1, ~ ..., \rvu_M ]| ~ w, \vtheta=[\vtheta_1, ..., \vtheta_M ]) \nonumber \\
    &= \prod_{m=1}^{M} \text{Bi}(\rv_{m} ~ | ~ w - \sum_{i<m}\rv_i, {\eta}_{\rv_{m}}) ~
    \text{Mu}(\rvu_{m}~ | ~ \rv_{m}, \vlambda_{{m}}), \nonumber \\
    &\text{where: } \eta_{\rv_m } = \frac{\1^{T}~ \vtheta_m}{1-\sum_{i<m} \1^{T}~ \vtheta_i}, ~~
    \vlambda_{{m}} = \frac{\vtheta_m}{\1^{T}~ \vtheta_m}.
\end{align}
Here, the probability of binomial, $\eta_{\rv_m}$, is the fraction of the remaining probability mass that is allocated to $\rv_m$, \ie the sum of all weights in the $m$-th group.
The probability vector (parameter) $\vlambda_{{m}}$ is the normalized multinomial probabilities of all count variables in the $m$-th group.
Intuitively, this decomposition of multinomial distribution can be viewed as a recursive \textit{stick-breaking} process where at each step, first a fraction of the remaining probability mass is allocated to a group by a binomial distribution
and then this fraction is distributed among that group's members by a multinomial distribution.
\end{theorem}
\begin{proof} Refer to appendix \ref{apdx:proof_mn2bnmn} for the proof. \end{proof}

Now, let's denote the group of candidate edges connecting the node $v_t(\pg_i^l)$ to the already generated graph, $\pg_{i,t}^l$, by $\gE_t({\pg_{i}^l})$ and their weights by the random vector $\rvu_{t} := [w_e]_{e ~\in~ \gE_t({\pg_{i}^l})}$ (the $t$-th row of the lower triangle of adjacency matrix $\hat{\Amat}_{i}^l$). 
Based on theorem \ref{thm:mn2bnmn}, at the $t$-th step of generating a partition graph, the probability of $\rvu_{t}$ can be characterized by the product of a binomial and a multinomial distribution.
This process is illustrated  in figure \ref{fig:MN_BN}.
We further increase the expressiveness of the generative network by extending this probability to a mixture model with $K$ mixtures:
\begin{align}
    p(\rvu_{t} ) &= \sum_{k=1}^{K}  \vbeta_k^l \text{Bi}(\rv_{t}  | w^{l-1}_{ii} - \sum_{i<t}\rv_i , {\eta}_{t, k}) 
    \text{Mu}(\rvu_{t}~ | \rv_{t}, \vlambda_{{t}, k})
    \label{eq:mix_mn_pg}\\
    \vlambda_{t, k} &=
    \softmax \left(  \mathrm{MLP}_{\vtheta}^l \big( \left[ \Delta \vh_{\gE_t({\pg_{i}^l})};~ h_{\parents(\pg_{i}^l)} \right] \big)  \right)[k,:]
     \label{eq:p_mn} \\
    \eta_{t, k} &=
    \mathrm{sigmoid} \left( \mathrm{MLP}_{\eta}^l \big( \left[
    \mathrm{pool}(\vh_{\pg_{i, t}^l});~ h_{\parents(\pg_{i}^l)} \right] \big)  \right)[k]
    \nonumber\\
    \vbeta^{l} &= \softmax \left( \mathrm{MLP}_{\beta}^l \big( \left[ \mathrm{pool}(\vh_{\pg_{i, t}^l} );~ h_{\parents(\pg_{i}^l)} \right]\big)  \right) \nonumber %
\end{align}

Where $\Delta \vh_{\gE_t({\pg_{i}^l})}$  is a $ |\gE_t({\pg_{i}^l})| \times d_h$ dimensional matrix composed of the set of edge representations  $ \{\Delta h_{<t, s>} := h_t - h_s ~|~ \forall ~ <t,s>~ \in \gE_t({\pg_{i}^l}) \}$, and
$\vh_{\pg_{i, t}^l}$  is a $ t \times d_h$ matrix of node representations of already generated nodes in the partition graph.
As explained in section \ref{sec:graph_gen}, the node representations are learned by a GNN network.
The graph level representation is obtained by the $\mathrm{addpool()}$ aggregation function.
Here, $\mathrm{MLP}_{\vtheta}^l()$ acts at the edge level and produce $K \times |\gE_t({\pg_{i}^l})|$ dimensional output,
while both $\mathrm{MLP}_{{\eta}_{\rv}}^l()$ and $ \mathrm{MLP}_{\beta}^l()$ produce $K$ dimensional arrays for $K$ mixture models.
All of the $\mathrm{MLP}$ networks are build by two hidden layers with $\mathrm{ReLU}$ activation functions and the mixture weights are denoted by  $\vbeta^{l}$.

For each partition graph $\pg_{i}^l$, node representation of its parent node $ h_{\parents(\pg_{i}^l)} $, \ie the node that represent $\pg_{i}^l$ at the higher level, is used as the context and concatenated to the representation matrices. %
The operation $\big[ \vx; ~y \big]$ denotes the concatenation of each row of matrix $\vx$ with vector $y$.
While generating a local component, this context enriches the node/edge representations by capturing long range interactions and
encoding the global structure of the graph.

On the other hand, the generation of edges in bipartite graph $ \bp_{ij}^l $ can be simply performed simultaneously, and we similarly use the mixture of multinomial distribution \eq{eq:mn_bp} to model the generative probability:
\begin{align}
    &p(\rvw := \gE({\bp_{ij}^l})) = \sum_{k=1}^{K} \vbeta_k^l  \text{Mu}(\rvw~|~w^{l-1}_{ij}, \vtheta^{l}_{ij, k} )
    \nonumber\\
    &\vtheta^{l}_{ij, k} = \softmax \left(  \mathrm{MLP}_{\vtheta}^l ( \left[ \Delta \vh_{\gE({\bp_{ij}^l})};~ \Delta h_{\parents(\bp_{ij}^l)}\right])   \right) [k,:]
    \label{eq:p_mn_bp}\\
    &\vbeta^{l} = \softmax \left( \mathrm{MLP}_{\beta}^l \big( \left[\mathrm{pool}( \Delta \vh_{\gE({\bp_{ij}^l})}  );~ \Delta h_{\parents(\bp_{ij}^l)} \right]\big)  \right) \nonumber
\end{align}
where $\Delta \vh_{\parents(\bp_{ij}^l)}$
is the edge representation of the parent edge of the bipartite graph, the edge that represent the $\bp_{ij}^l$ at the higher level.

In equations \eq{eq:p_mn} and \eq{eq:p_mn_bp},
the probability of the integer-valued edges are modeled by $\softmax()$ function
but since the final graphs in our experiments have binary edges weights, we instead use multi-hot activation function
$\sigma: \mathbb{R}^K \to (K-1)$-{simplex},
 defined as %
\[ \sigma(\mathbf{z})_i = \frac{\mathrm{sigmoid}({z_i})}{\sum_{j=1}^K \mathrm{sigmoid}({z_j})}.\]
for the leaf level while the upper levels still employ $\softmax$.
In our experiments, this function could better model the edge probability of the leaf level compared to standard $\softmax$ function.
As an alternative, we also modeled the edges at the leaf level by the mixture of Bernoulli using $\mathrm{sigmoid}()$ activation for the output while higher levels use mixture of multinomials.
A possible extension to this work could be using the cardinality potential model \citep{hajimirsadeghi2015visual}, derived to model the distribution over the set of binary random variables, for the last level. %

\textbf{Remark} Training and generation of the proposed hierarchical model is highly parallelizable and can be sped up to $\mathcal{O}( c \log n)$ sequential steps where $c$ the size of largest graph parts. %

\section{Related Work}\label{sec:related}

With recent progress in graph neural networks, several deep neural network models have been introduced \citep{de2018molgan, simonovsky2018graphvae, kipf2016semi, ma2018constrained, liu2019graph} that are base on variational autoencoders \citep{kingma2013auto},
But these methods are weak in capturing the complex dependencies in graph structures and thus quality of graph generation degrades as graphs become moderate or large in size \cite{li2018learning}.

Autoregressive deep architectures, on the other hand, model graph generation as a sequential decision making process.
\citet{li2018learning} proposed generative model based on GNN but it has high complexity of $\mathcal{O}(mn^2)$.
GraphRNN \citep{you2018graphrnn} models graph generation with a two-stage RNN architecture, with the first RNN generating new nodes and the second generating links of the new nodes.
It thus has to traverses all elements of the adjacency matrix in a predefined order, resulting in $\mathcal{O}(n^2)$ thus not scalable to large graphs.
On the other hand, GRAN \citep{liao2019GRAN} uses graph attention networks and improves the complexity by generating the adjacency matrix in a row-by-row fashion, \ie generating all the edges between the new node and already generated graph in one step, resulting in $\mathcal{O}(n)$ recursive steps.
In an attempt to improve the scalability of generative models for graph, \citet{dai2020scalableGraphGen} proposed an algorithm for sparse graphs that reduce the training complexity to $\mathcal{O}(\log n)$ while its generation time is increased to $\mathcal{O}((n+m)\log n)$ complexity and its recursive generation process does not incorporate community structure of the graph.

In explicitly dealing with hierarchical structures, \citet{jin2020hierarchicalMolec} proposed a generation method for molecular graphs that recursively selects motifs, the basic building blocks, from a set and predicts the attachment of that motif to emerging molecule.
This model require prior domain-specific knowledge and relies on molecule-specific graph motifs. Moreover, graphs are abstracted in only two levels and component generation cannot be performed in parallel.
A hierarchical normalizing flow model for molecular graphs is proposed in \citet{de2018molgan} that generates new molecules from a single node by recursively dividing every node into two nodes.
Merging and splitting of pair of nodes in this model is based on the the node's neighborhood so it does not include the diverse community structure of the graphs and hence its hierarchical generation is structurally limited. %

\section{Experiments}\label{sec:exp}
In our empirical studies,  we compare the proposed method against some well-established baselines on two synthetics datasets and three real-world datasets.

\textbf{Datasets:}
First,
we generated \textbf{R}elaxed \textbf{C}aveman \textbf{G}raphs (\textbf{RCG}) which starts with $7\le l<25$ cliques of size $15\le k<25$. Edges are then randomly rewired with probability $p=1/l$ to different cliques.
We also generated \textbf{P}lanted \textbf{P}artition \textbf{G}raphs (\textbf{PPG}).
This model partitions a graph with $n$ nodes in $20\le l<30$ groups with $15\le k<25$ nodes each.
Nodes of the same group are linked with a probability $ p_{in}=.75$, and nodes of different groups are linked with probability $p_{out}=10 / (k l^2) $. Both of these datasets that exhibit strong community structures are generated using \textsc{NetworkX} Python package \citep{NetworkX}.

The real-world datasets are
(1) \textbf{Protein} dataset which contains 918 protein graphs, each of which has 100 to 500 nodes for amino acids and has edges for amino acid pairs closer than 6 Angstroms \citep{dobson2003distinguishing},
(2) \textbf{Ego} dataset which contains 757 3-hop ego networks with 50 to 300 nodes extracted from the CiteSeer dataset, with nodes representing documents and edges representing citation relationships \citep{sen2008collectiveEgo},
and
(3) \textbf{Point Cloud} with 41 simulated 3D point clouds of household objects. This dataset has about 1.4k nodes on average with maximum of over 5k nodes. Each point is mapped to a node and edges connecting the k-nearest neighbors in Euclidean distance in 3D space are added to the graphs \citep{neumann2013graph}.

To partition graphs and obtain hierarchical graph structures, we applied Louvain algorithm on all of these datasets.
This resulted in hierarchical graphs of depth $L=2$ for the synthetic datasets, while for the real world graphs it produced at least 3 levels so we spliced out the intermediate levels so that all have equal depth of $L=3$.\footnote{The proposed architecture can be trained on HGs with uneven heights by adding empty graphs at the root levels of those HGs with lower height so that they are not sampled during the training.}
Before training the models, we follow the protocol in \cite{liao2019GRAN} to randomly create a 80$\%$-20$\%$ training-testing split, with 20$\%$ of the training data reserved as the validation set.

\begin{table*}[t]
\begin{center}
\caption{ \footnotesize
In this table the quality of generated graphs are compared in terms of the MMD of graph degree distributions (\textit{Deg.}),  clustering coefficient (\textit{Clus.}),
 4-node orbits (\textit{Orbit}), and  the spectra of the graph Laplacian (\textit{Spec}.).
For all the metrics, the smaller the better.
Graph sizes, $(|V|_{max}, |V|_{avg}, |E|_{max}, |E|_{avg})$, are listed for each dataset.
}
\label{table:res}
	\begin{minipage}{.99\linewidth}
		\centering
    \begin{adjustbox}{width=1.00\textwidth,}
    \begin{tabular}{
l ||cccc|cccc|cccc|cccc|cccc}
                    & \multicolumn{4}{c}{\textbf{Protein}} & \multicolumn{4}{c}{\textbf{3D Point Cloud}} & \multicolumn{4}{c}{\textbf{Ego}} & \multicolumn{4}{c}{\textbf{PPG}} & \multicolumn{4}{c}{\textbf{RCG}} \\
                    & \multicolumn{4}{c}{(500, 258, 1575, 646)} & \multicolumn{4}{c}{(5.03k, 1.4k, 10.9k, 3k)}  & \multicolumn{4}{c}{(399, 144, 1062, 332)}  & \multicolumn{4}{c}{(696, 477, 7.5k, 4.4k)}  & \multicolumn{4}{c}{(576, 261, 6.6k, 2.2k)} \\
                     & Deg.     & Clus.    & Orbit  & Spec.    & Deg.           & Clus. & Orbit  & Spec.    & Deg.     & Clus. & Orbit    & Spec.    & Deg.     & Clus.  & Orbit   & Spec.   & Deg.   & Clus.    & Orbit  & Spec.   \\ \toprule
\textbf{Erdos-Renyi} & 5.64$e^{-2}$ & 1        & 1.54   & 9.13$e^{-2}$ & 3.1$e^{-1}$           & 1.22  & 1.27   & 4.26$e^{-2}$ & 1.6$e^{-1}$     & 9.4$e^{-1}$  & 8.5$e^{-1}$     & 1.8$e^{-1}$     & 2.83$e^{-1}$    & 1.04   & 1.94$e^{-1}$   & 2.01$e^{-1}$   & 1.7$e^{-1}$   & 8.0$e^{-1}$      & 1.2$e^{-1}$   & 2.0$e^{-1}$     \\
\textbf{GraphVAE}   & 4.8$e^{-1}$     & 7.14$e^{-2}$ & 7.4$e^{-1}$   & 1.1$e^{-1}$     & -              & -     & -      & -        & -          & -       & -         &  -        &  -        &  -      &  -       &   -      &     -   &    -      & -       &  -       \\
\textbf{GraphRNN-S}  & 4.02$e^{-2}$ & \textbf{4.79$e^{-2}$} & 2.3$e^{-1}$   & 2.1$e^{-1}$     &-&-&-&-&   6.51$e^{-3}$ & 2.24$e^{-1}$ & 6.35$e^{-2}$ &  7.30$e^{-2} $ &    4.34$e^{-2}$ &  3.01$e^{-1}$ & 3.32$e^{-2}$ & 1.70$e^{-2}$      &     7.02$e^{-2}$   &    2.47$e^{-2}$      &  3.41$e^{-2}$      &  4.91$e^{-2}$     \\
\textbf{GraphRNN}    & 1.06$e^{-2}$ & 1.4$e^{-1}$     & 8.8$e^{-1}$   & 1.88$e^{-2}$ & -              & -     & -      & -        & 2.44$e^{-2}$ &3.46$e^{-1}$ &1.35$e^{-1}$ &8.91$e^{-2}$       &9.65$e^{-2}$ &3.12$e^{-1}$ &2.97$e^{-2}$ &4.90$e^{-2}$      &   6.74$e^{-2}$ &1.82$e^{-2}$ &3.00$e^{-2}$ & 4.93$e^{-2}$        \\
\textbf{GRAN}        & \textbf{1.98$e^{-3}$} & {4.86$e^{-2}$ } & 1.3$e^{-1}$   & \textbf{5.13$e^{-3}$} & \textbf{1.75$e^{-2}$}       & 5.1$e^{-1}$  & 2.1$e^{-1}$   & \textbf{7.45$e^{-3}$} & 3.2$e^{-2}$  & 1.7$e^{-1}$  & 2.6$e^{-2}$  & 4.6$e^{-2}$  & 5.67$e^{-2}$   & 2.3$e^{-1}$   & 2.82$e^{-1}$   & {1.71$e^{-2}$ } & 7.50$e^{-2}$  & 1.34$e^{-2}$   & 9.95$e^{-2}$ & {5.70$e^{-2}$}   \\
\hline 
\rule{0pt}{3ex}
\textbf{MRG-B}       & 5.1$e^{-3}$  & 6.27$e^{-2}$ & 1.08$e^{-1}$ & 8.0$e^{-3}$   & 1.29$e^{-1}$          & \textbf{3.4$e^{-1}$} & 5.9$e^{-2}$  & 8.9$e^{-3}$   & \textbf{4.1$e^{-3}$}  & \textbf{6.2$e^{-2}$} & 1.8$e^{-2}$  & \textbf{1.42$e^{-2}$} & \textbf{4.79$e^{-3}$}  & \textbf{8.79$e^{-2}$} & 4.8$e^{-2}$   & \textbf{1.85$e^{-3}$} & \textbf{1.45$e^{-2}$} & \textbf{1.29$e^{-2}$} & 2.75$e^{-2}$ &\textbf{4.1$e^{-3}$} \\
\textbf{MRG}         & 6.49$e^{-3}$  & 2.24$e^{-1}$    & \textbf{5.78$e^{-2}$} & 1.31$e^{-2}$   & 2.07$e^{-1}$         & 8.06$e^{-1}$ & \textbf{2.75$e^{-2}$} & 2.24$e^{-2}$   & {1.78$e^{-2}$} & 2.24$e^{-1}$ & \textbf{1.16$e^{-2}$} & 2.07$e^{-2}$ & 1.51$e^{-1}$ & 3.68$e^{-1}$  & \textbf{7.75$e^{-3}$} & 1.93$e^{-2}$  & 4.45$e^{-2}$ & 2.60$e^{-2}$    & \textbf{1.15$e^{-2}$} & 5.76$e^{-2}$ \\
\hline
\end{tabular}

    \end{adjustbox}
	\end{minipage}
\end{center}
\end{table*}

\begin{figure*}[t]
\begin{center}
    \begin{tabular}{c c|c|c}
    \begin{minipage}{.05\linewidth}
    \begin{turn}{90} \hspace{5pt} MRG \hspace{15pt} GRAN \hspace{5pt} GraphRNN \hspace{10pt} Train \hspace{25pt}  \end{turn}
    \end{minipage} &\hfill
    \begin{minipage}{.28\linewidth}
    \begin{center} Ego \end{center}
    \adjincludegraphics[width=.32\linewidth,
		trim={0 {0\height} 0 {.66\height}},clip]{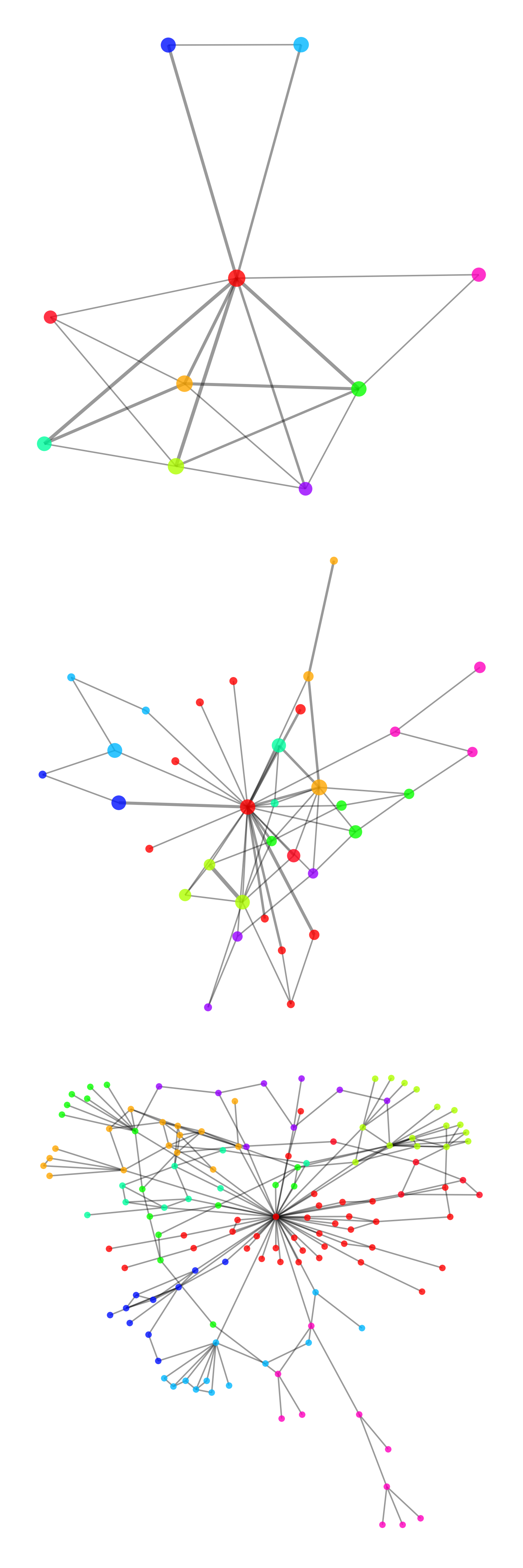}
    \adjincludegraphics[width=.32\linewidth,
		trim={0 {0\height} 0 {.66\height}},clip]{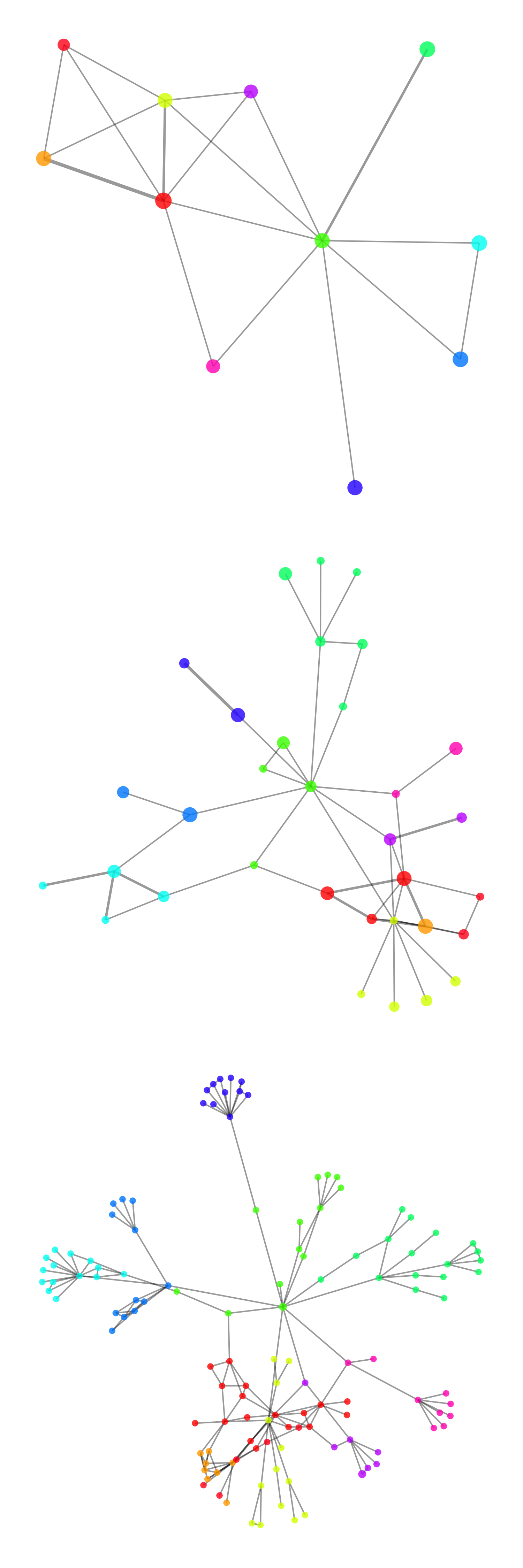}
    \adjincludegraphics[width=.32\linewidth,
		trim={0 {0\height} 0 {.66\height}},clip]{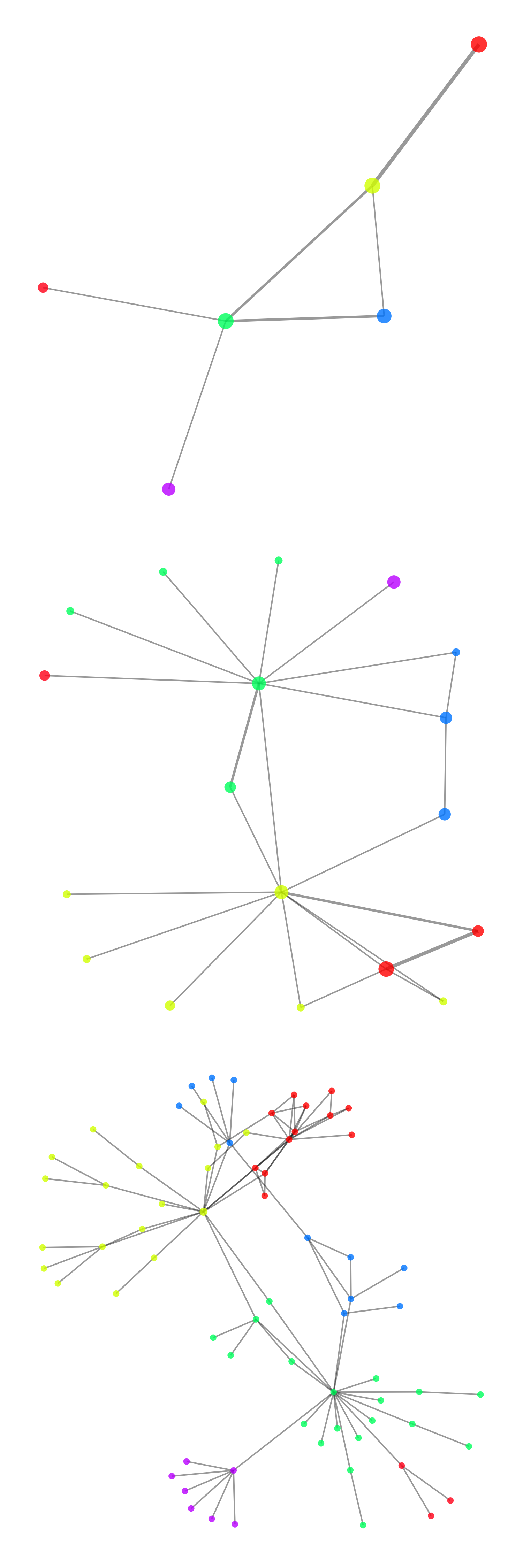}
    \\ \hrule
    \adjincludegraphics[width=.32\textwidth,
		trim={0 {0\height} 0 {0\height}},clip]{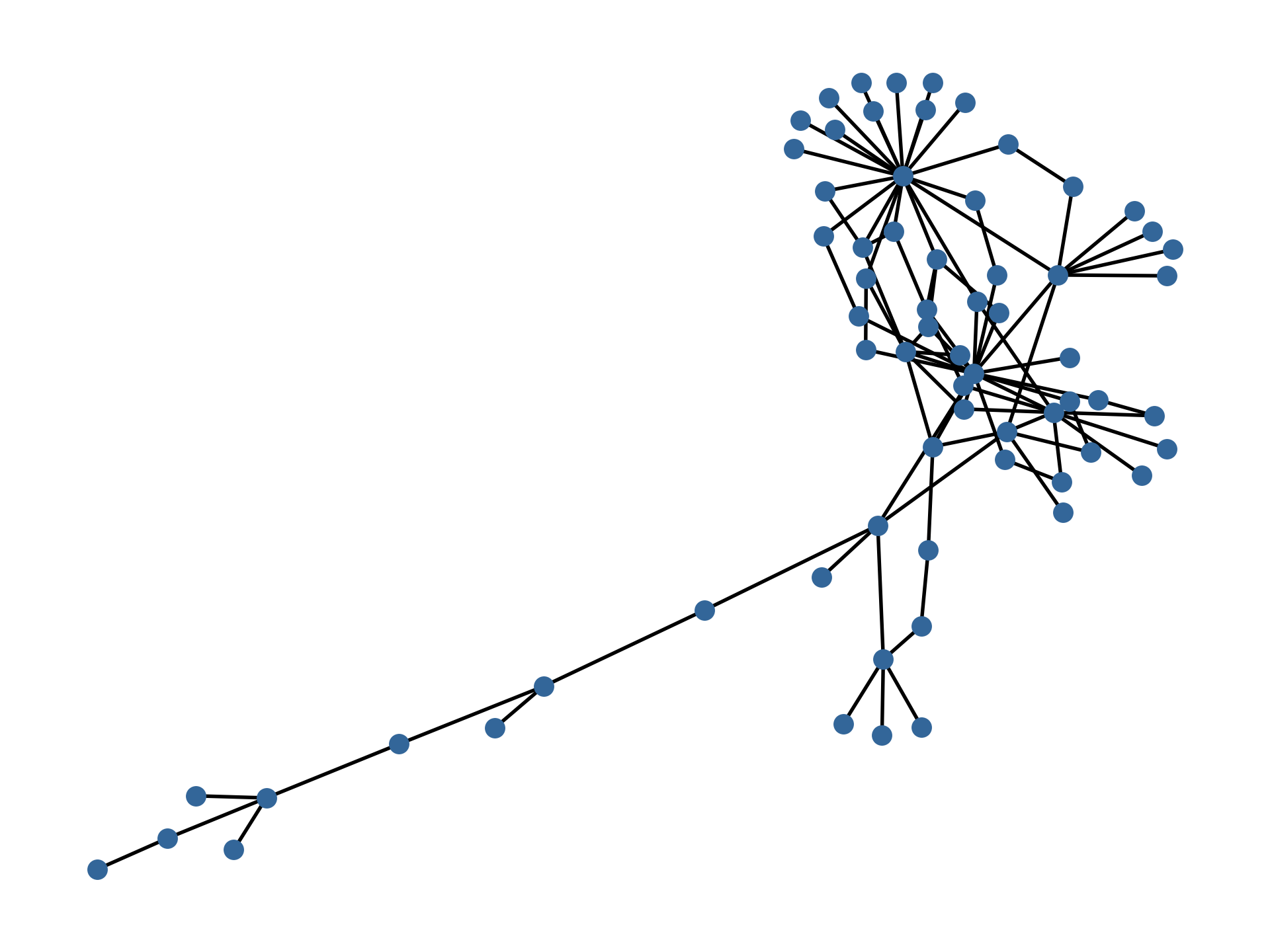}
    \adjincludegraphics[width=.32\textwidth,
		trim={0 {0\height} 0 {0\height}},clip]{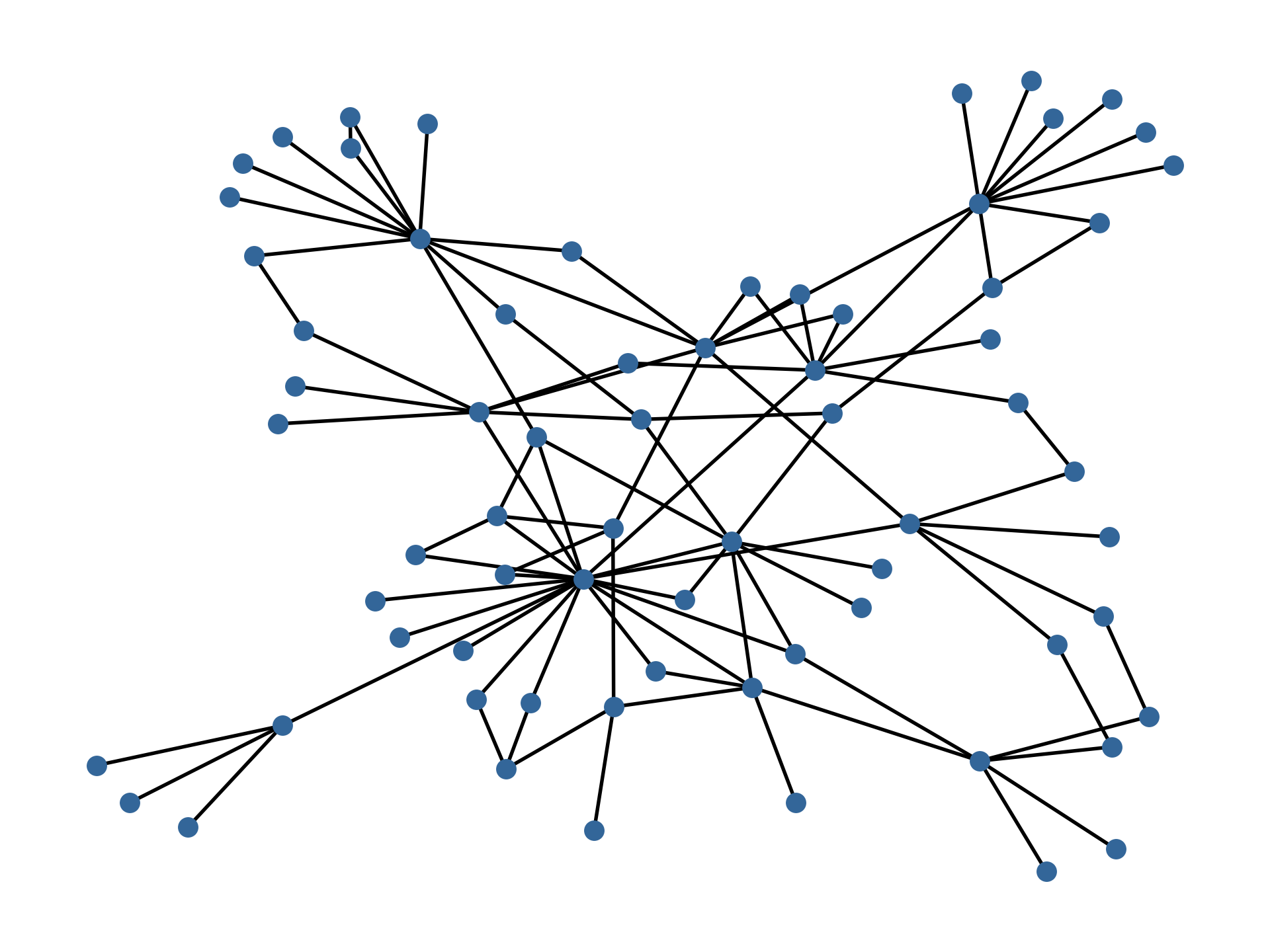}
    \adjincludegraphics[width=.32\textwidth,
		trim={0 {0\height} 0 {0\height}},clip]{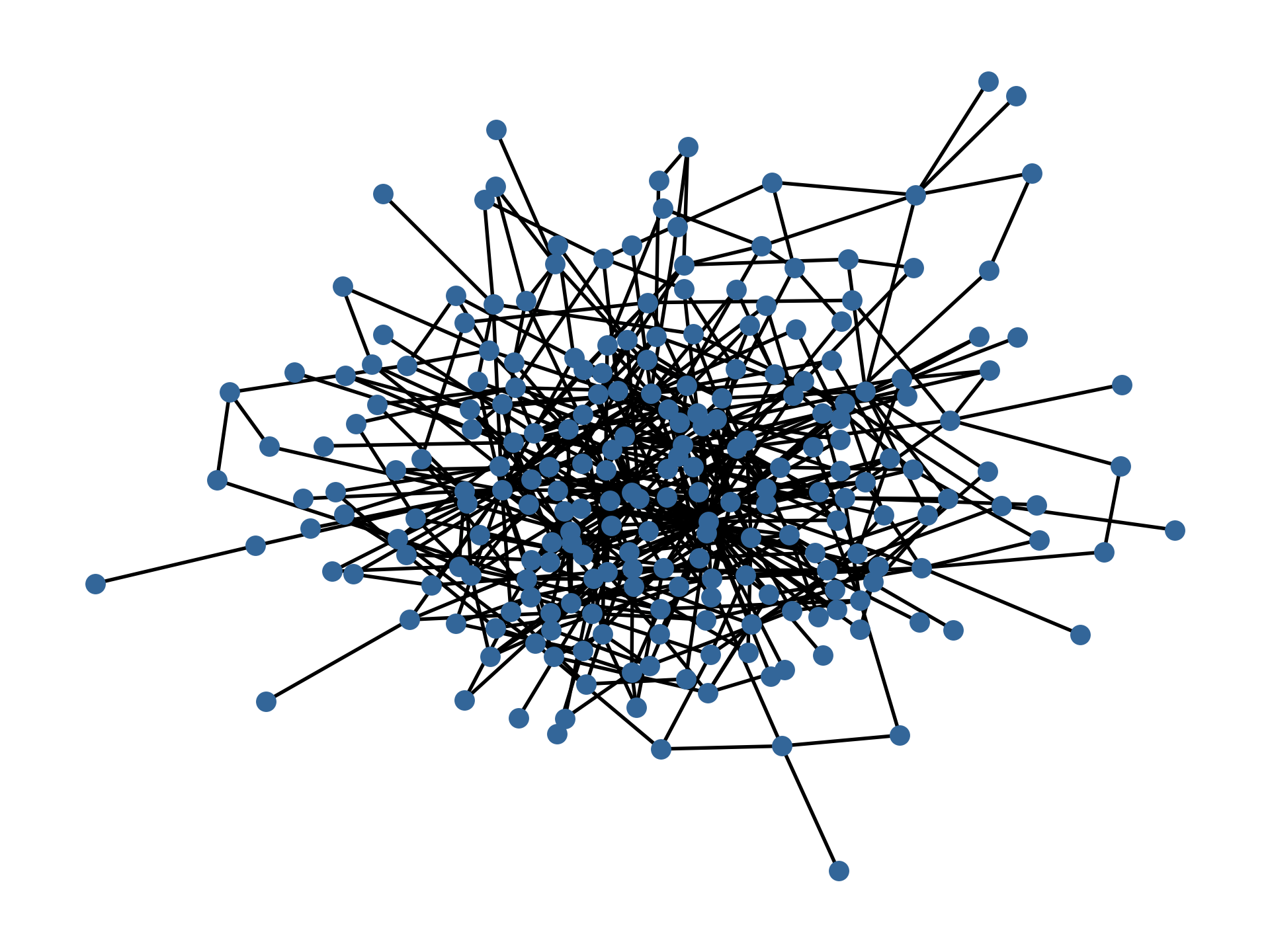}
  \\ \hrule
    \adjincludegraphics[width=.32\textwidth,
		trim={0 {0\height} 0 {0\height}},clip]{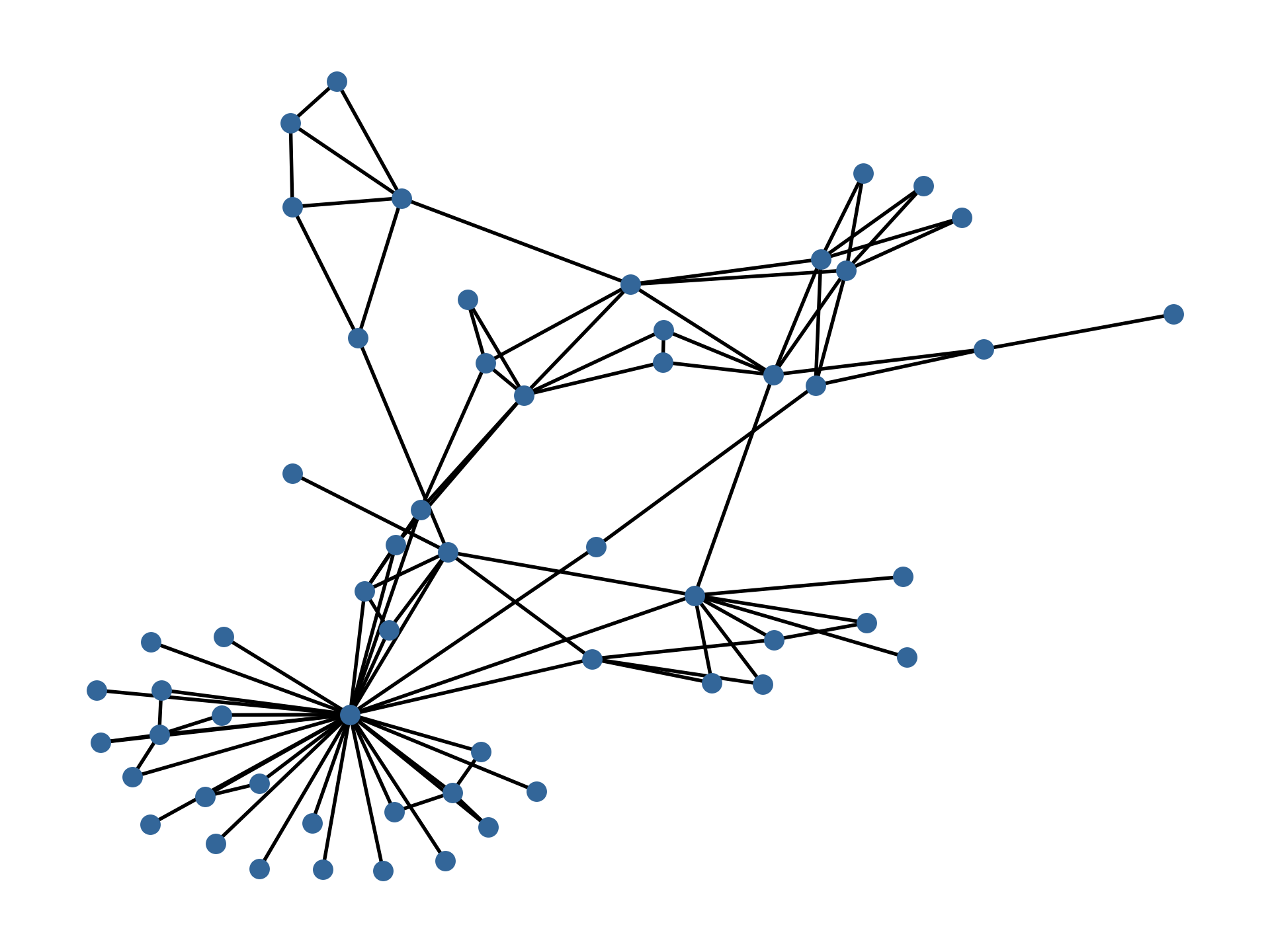}
    \adjincludegraphics[width=.32\textwidth,
		trim={0 {0\height} 0 {0\height}},clip]{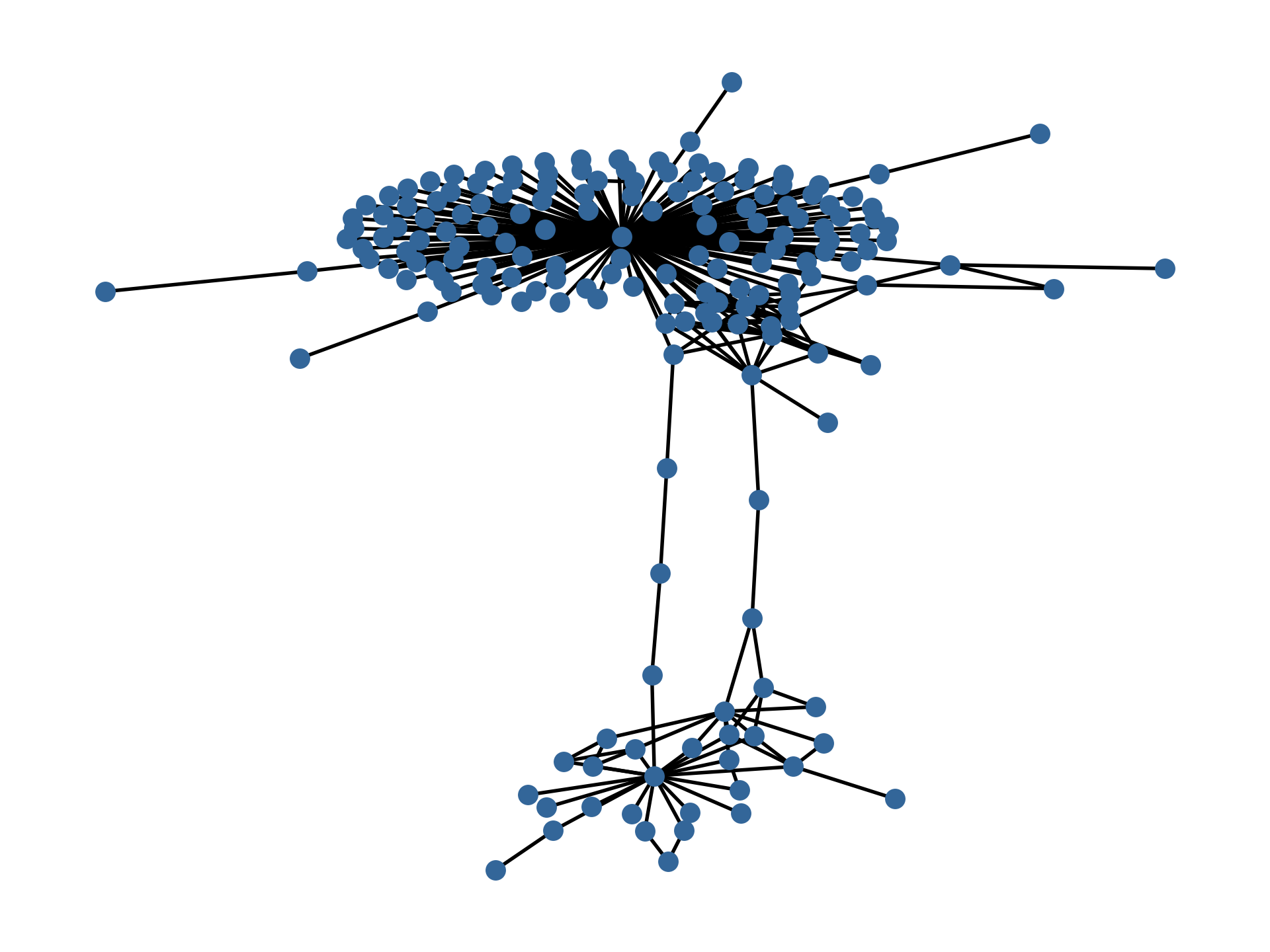}
    \adjincludegraphics[width=.32\textwidth,
		trim={0 {0\height} 0 {0\height}},clip]{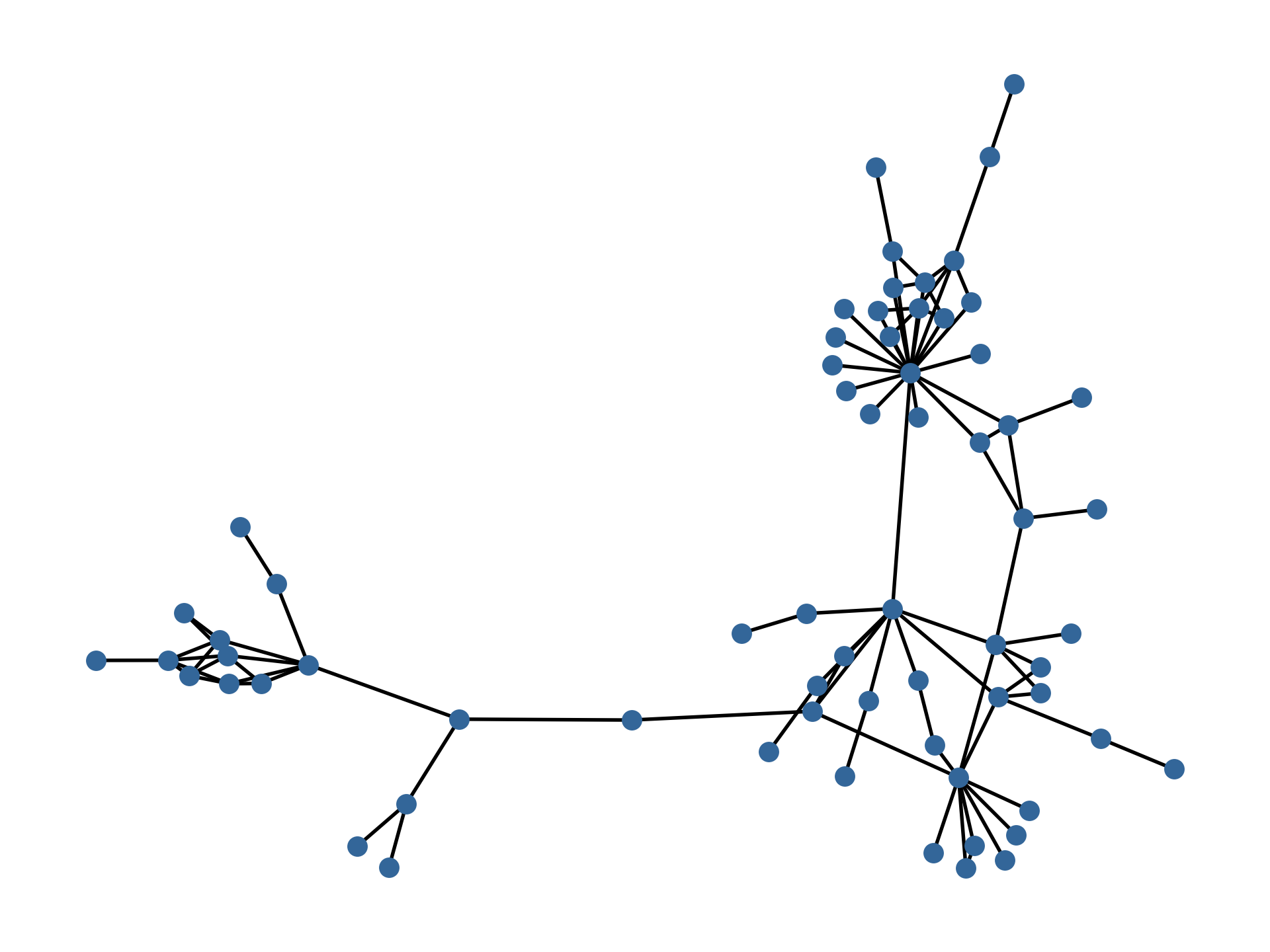}
  \\ \hrule
    \adjincludegraphics[width=.32\textwidth,
		trim={0 {0\height} 0 {.66\height}},clip]{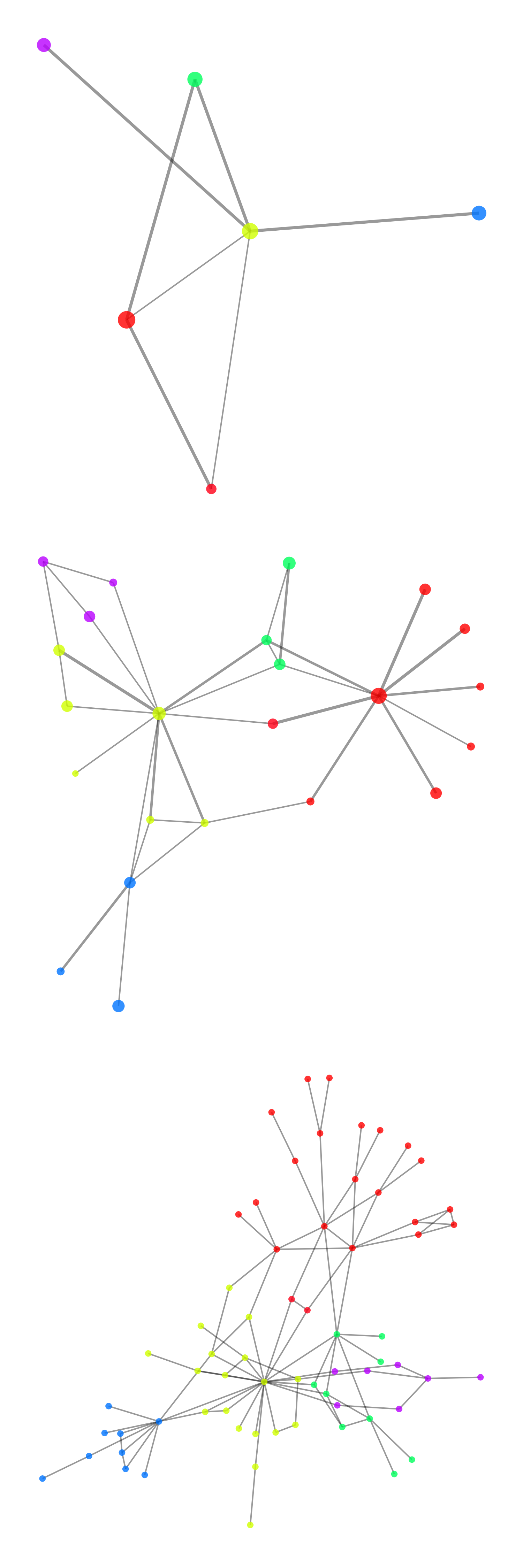}
	\adjincludegraphics[width=.32\textwidth,
		trim={0 {0\height} 0 {.66\height}},clip]{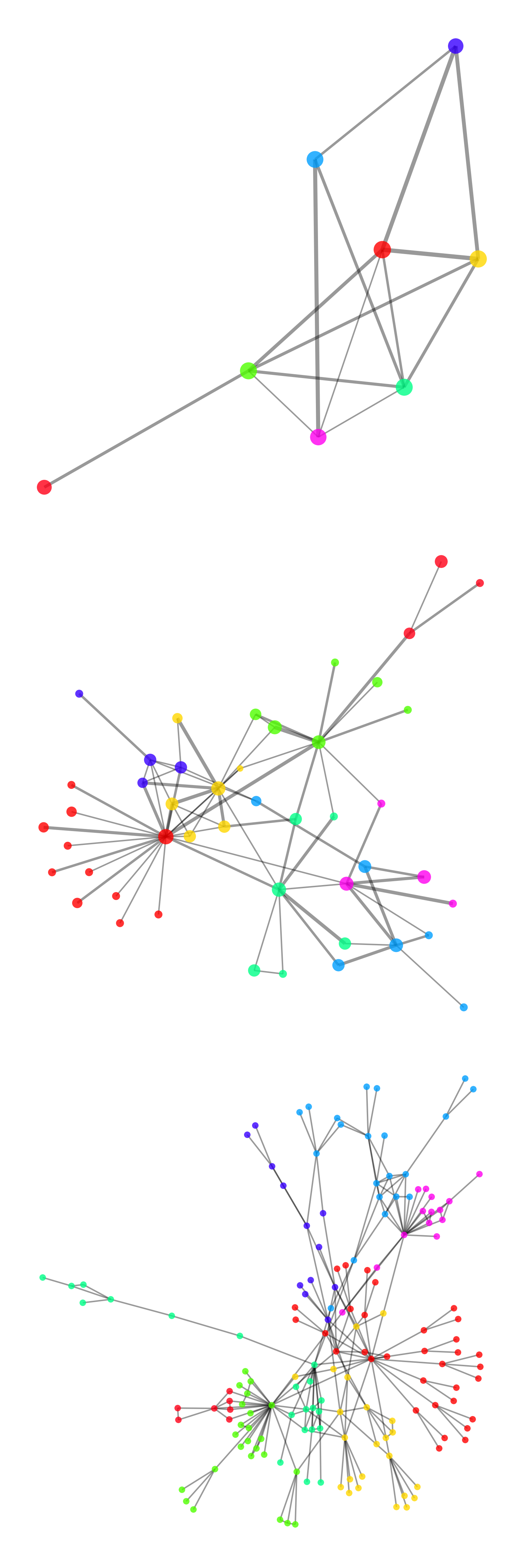}
	\adjincludegraphics[width=.32\textwidth,
		trim={0 {0\height} 0 {.66\height}},clip]{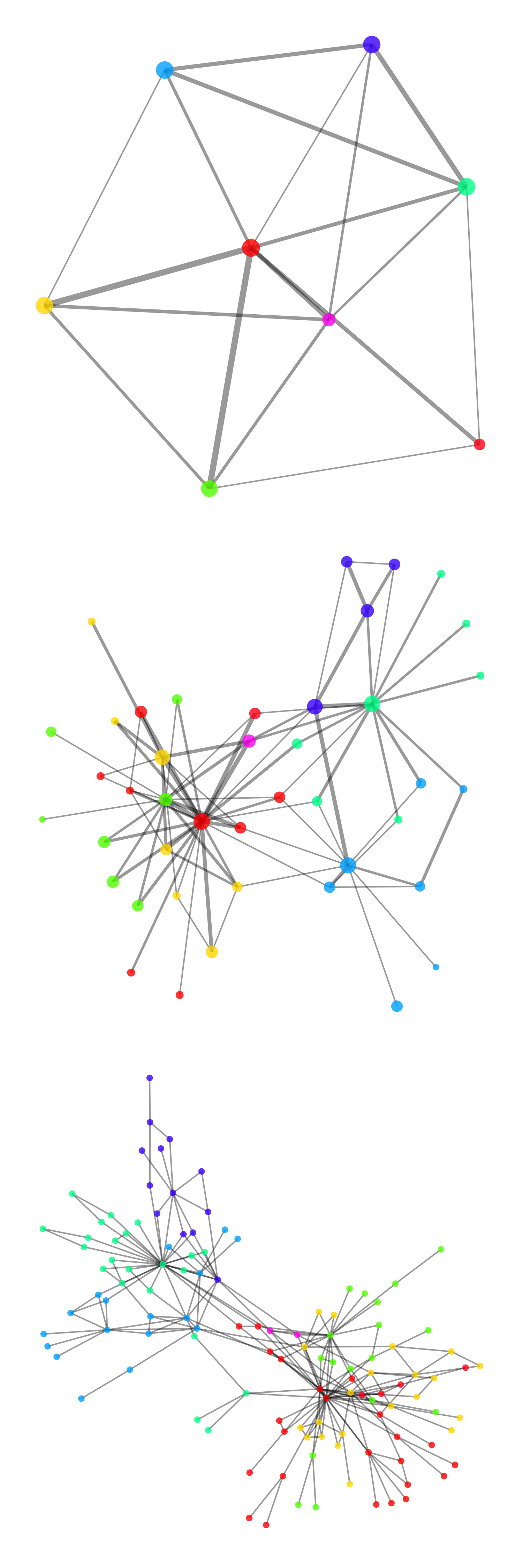}
    \end{minipage} &\hfill
    \begin{minipage}{.28\linewidth}
    \begin{center} Protein \end{center}
    \adjincludegraphics[width=.32\textwidth,
		trim={0 {0\height} 0 {.66\height}},clip]{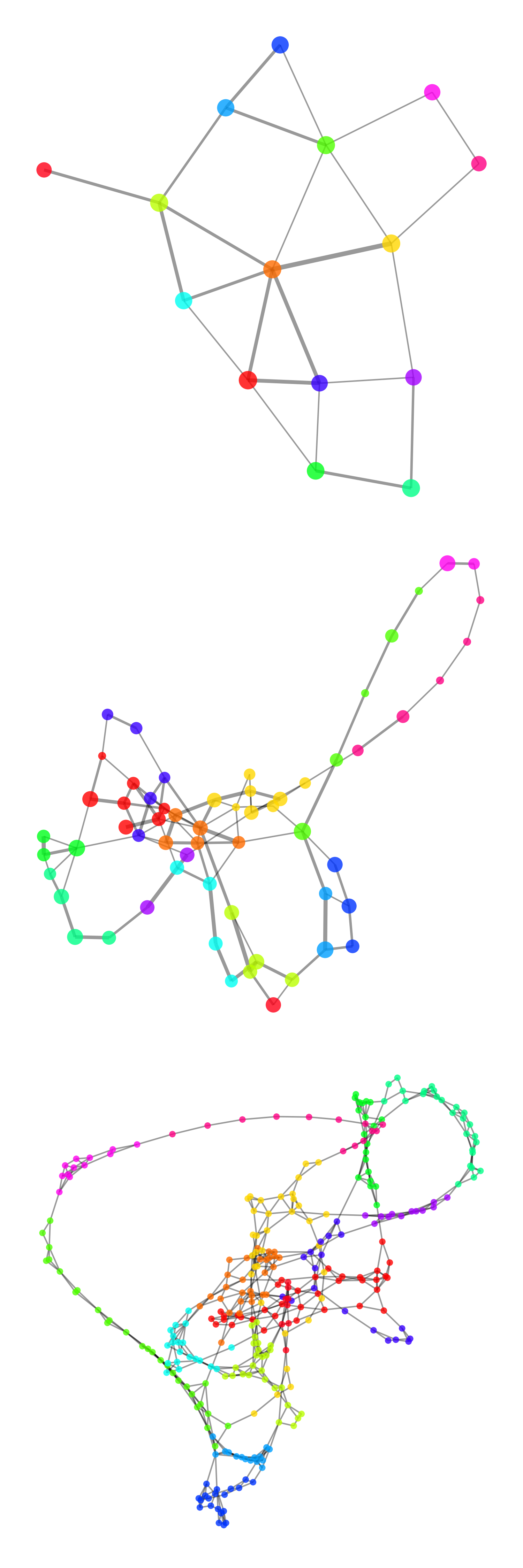}
    \adjincludegraphics[width=.32\textwidth,
		trim={0 {0\height} 0 {.66\height}},clip]{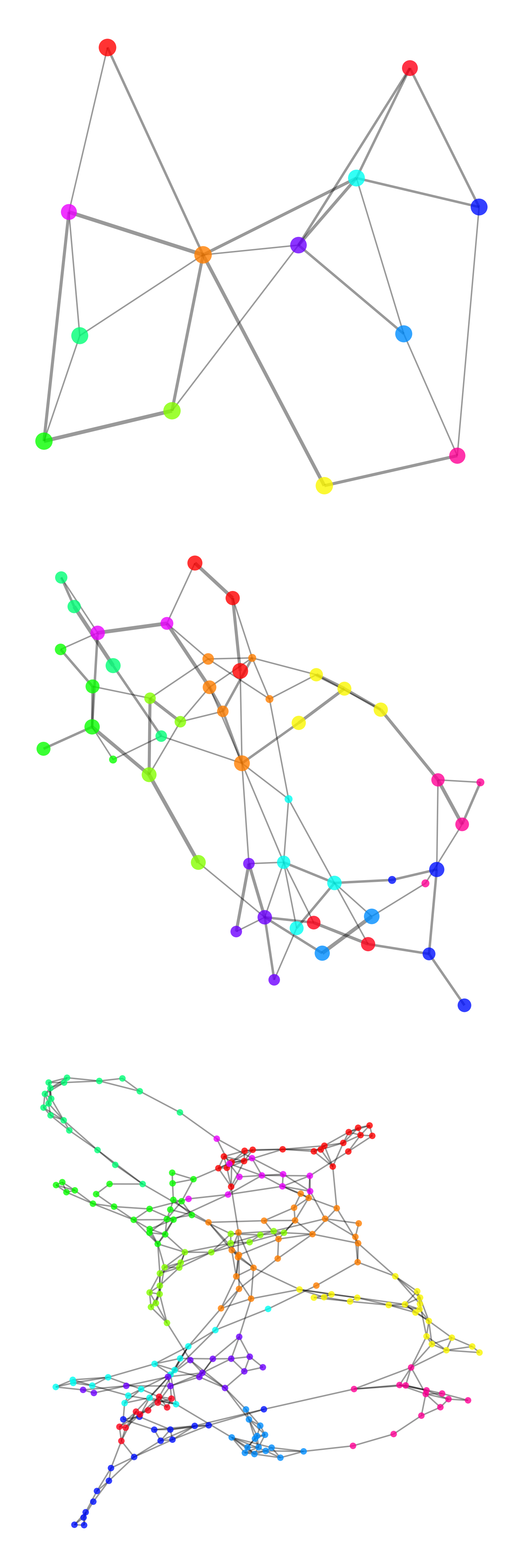}
    \adjincludegraphics[width=.32\textwidth,
		trim={0 {0\height} 0 {.66\height}},clip]{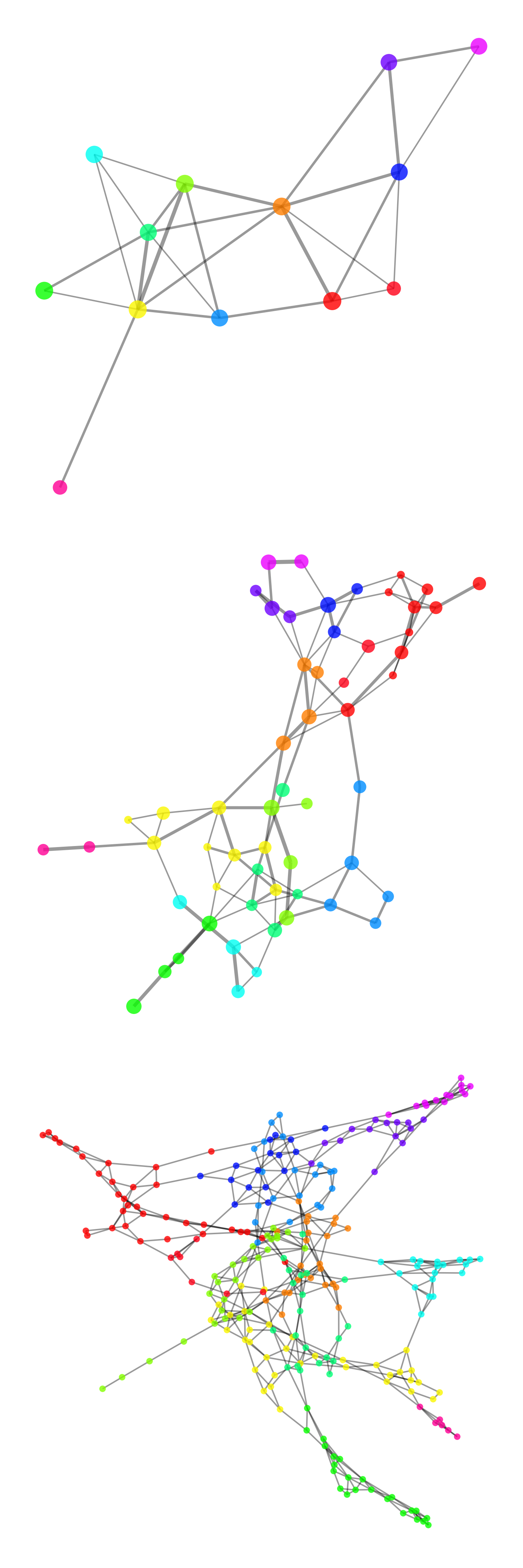}
    \\ \hrule
    \adjincludegraphics[width=.32\textwidth,
		trim={0 {0\height} 0 {0\height}},clip]{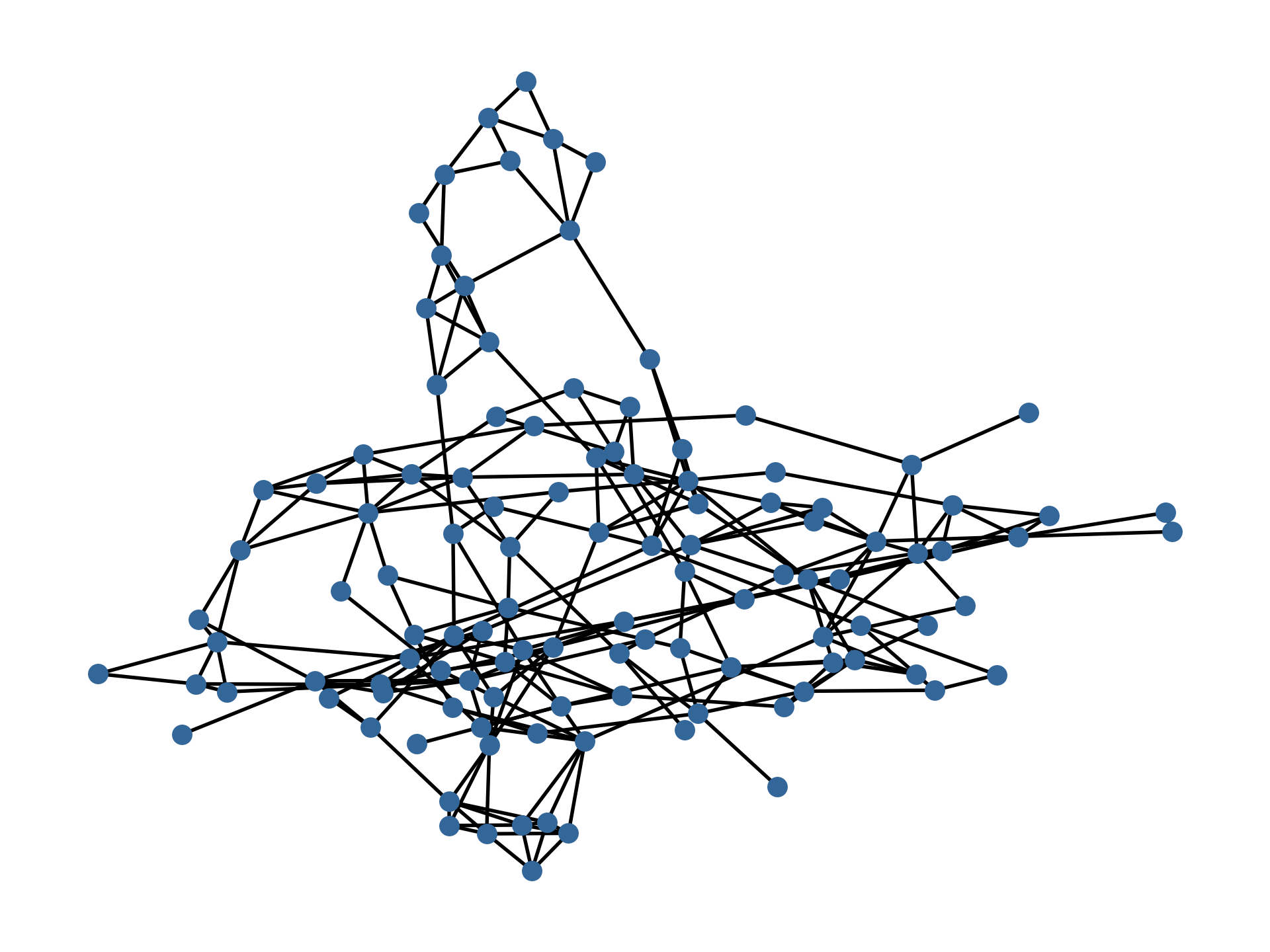}
    \adjincludegraphics[width=.32\textwidth,
		trim={0 {0\height} 0 {0\height}},clip]{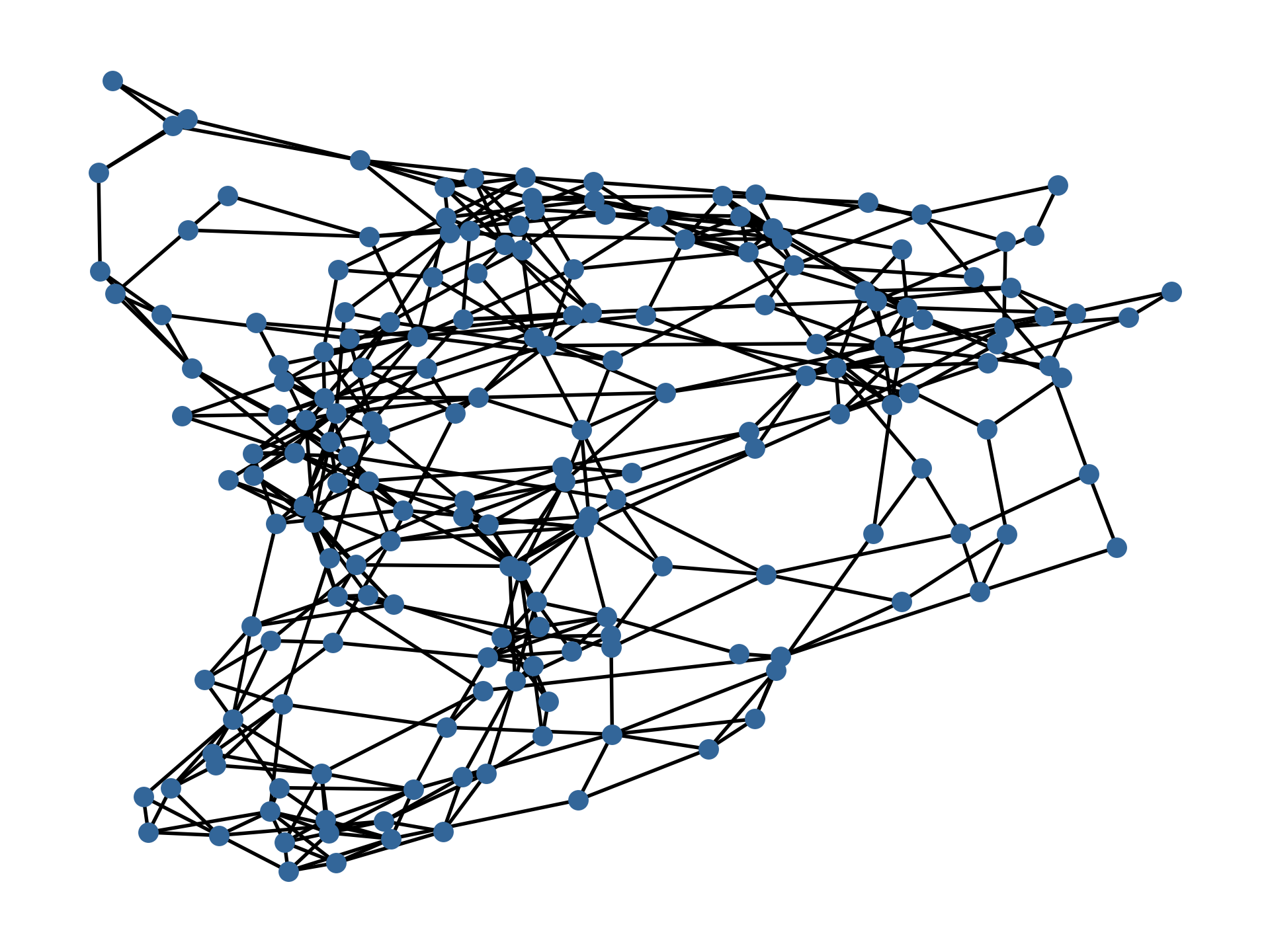}
    \adjincludegraphics[width=.32\textwidth,
		trim={0 {0\height} 0 {0\height}},clip]{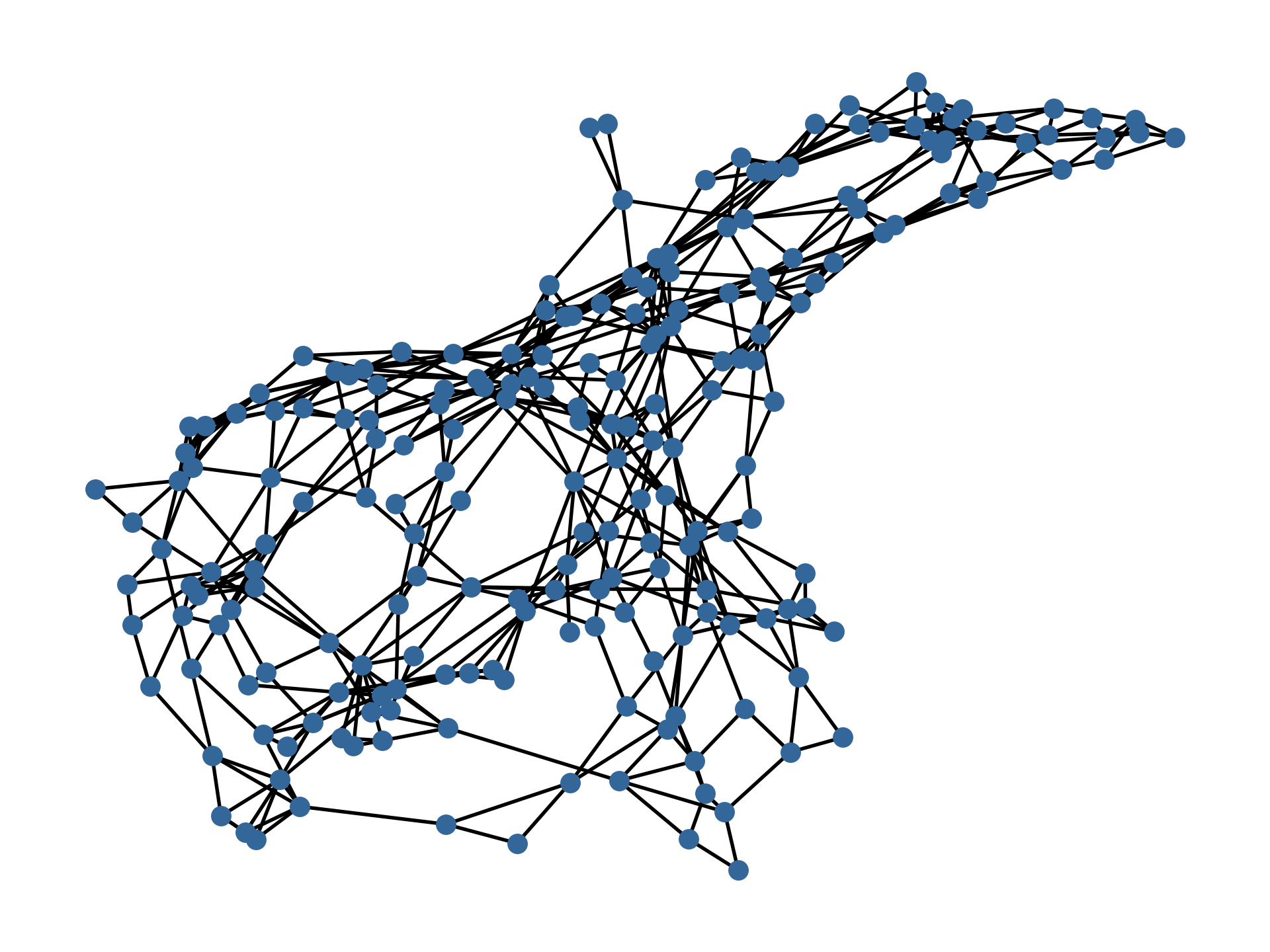}
  \\ \hrule
  \adjincludegraphics[width=.32\textwidth,
		trim={0 {0\height} 0 {0\height}},clip]{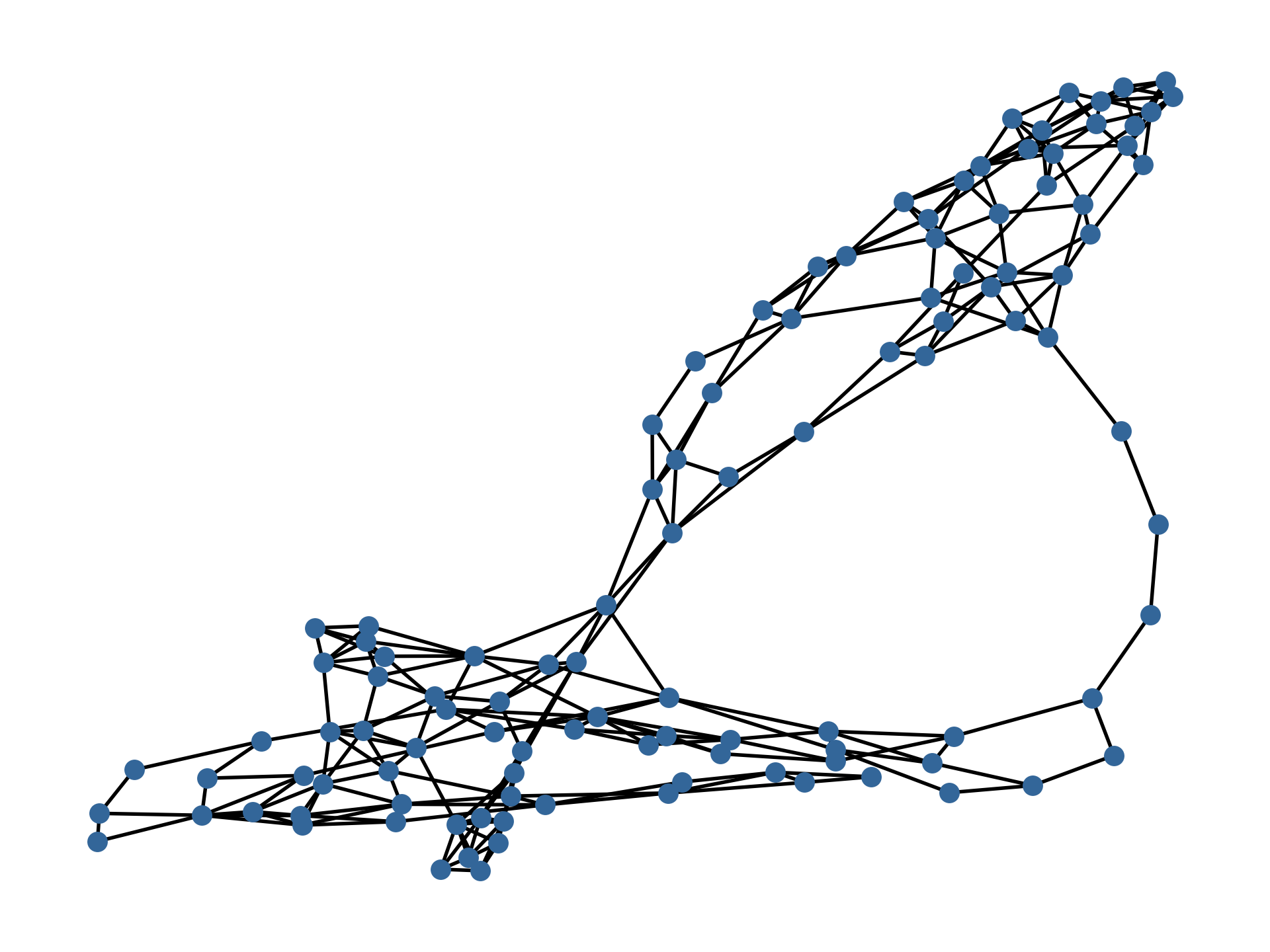}
    \adjincludegraphics[width=.32\textwidth,
		trim={0 {0\height} 0 {0\height}},clip]{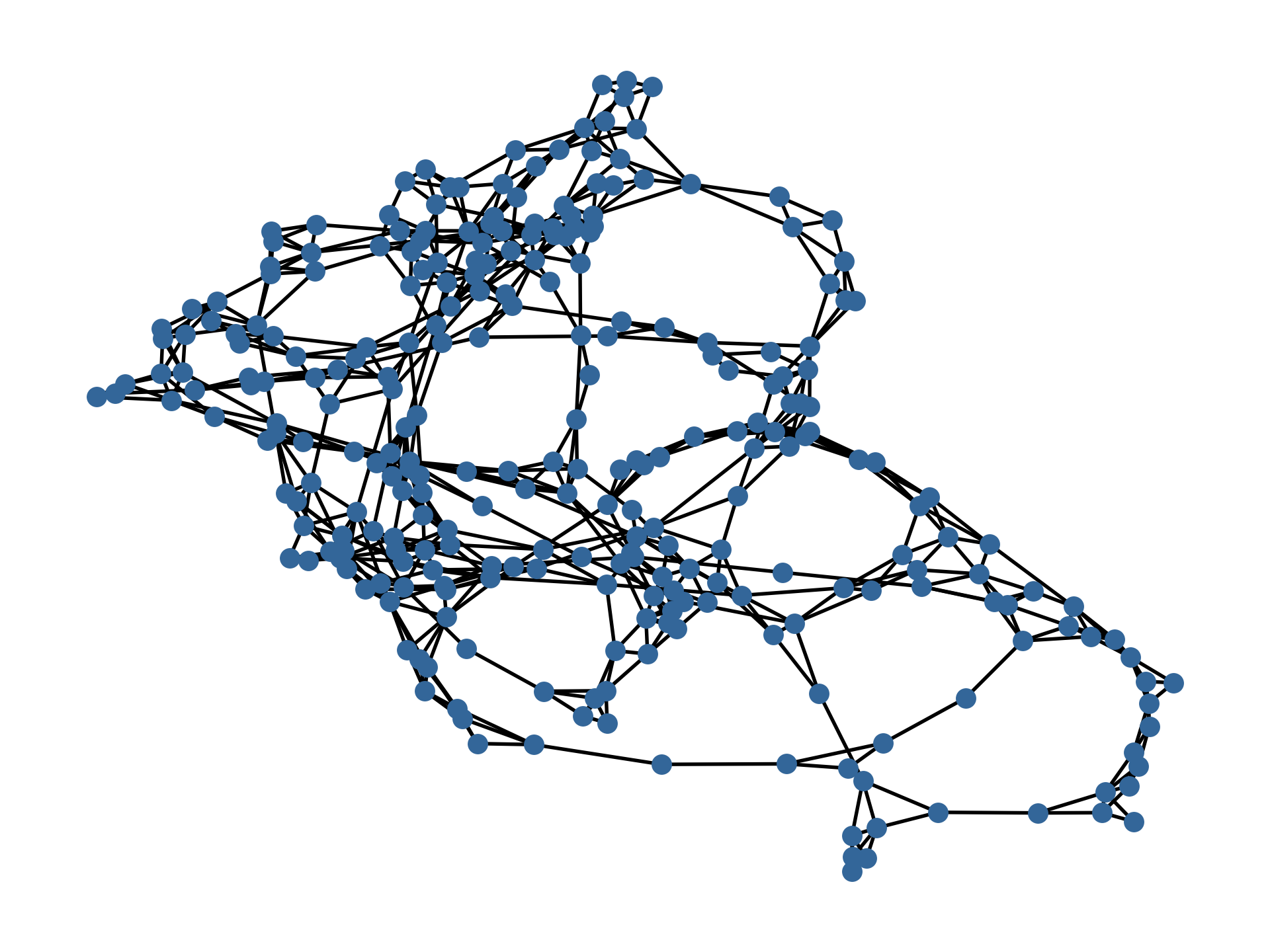}
    \adjincludegraphics[width=.32\textwidth,
		trim={0 {0\height} 0 {0\height}},clip]{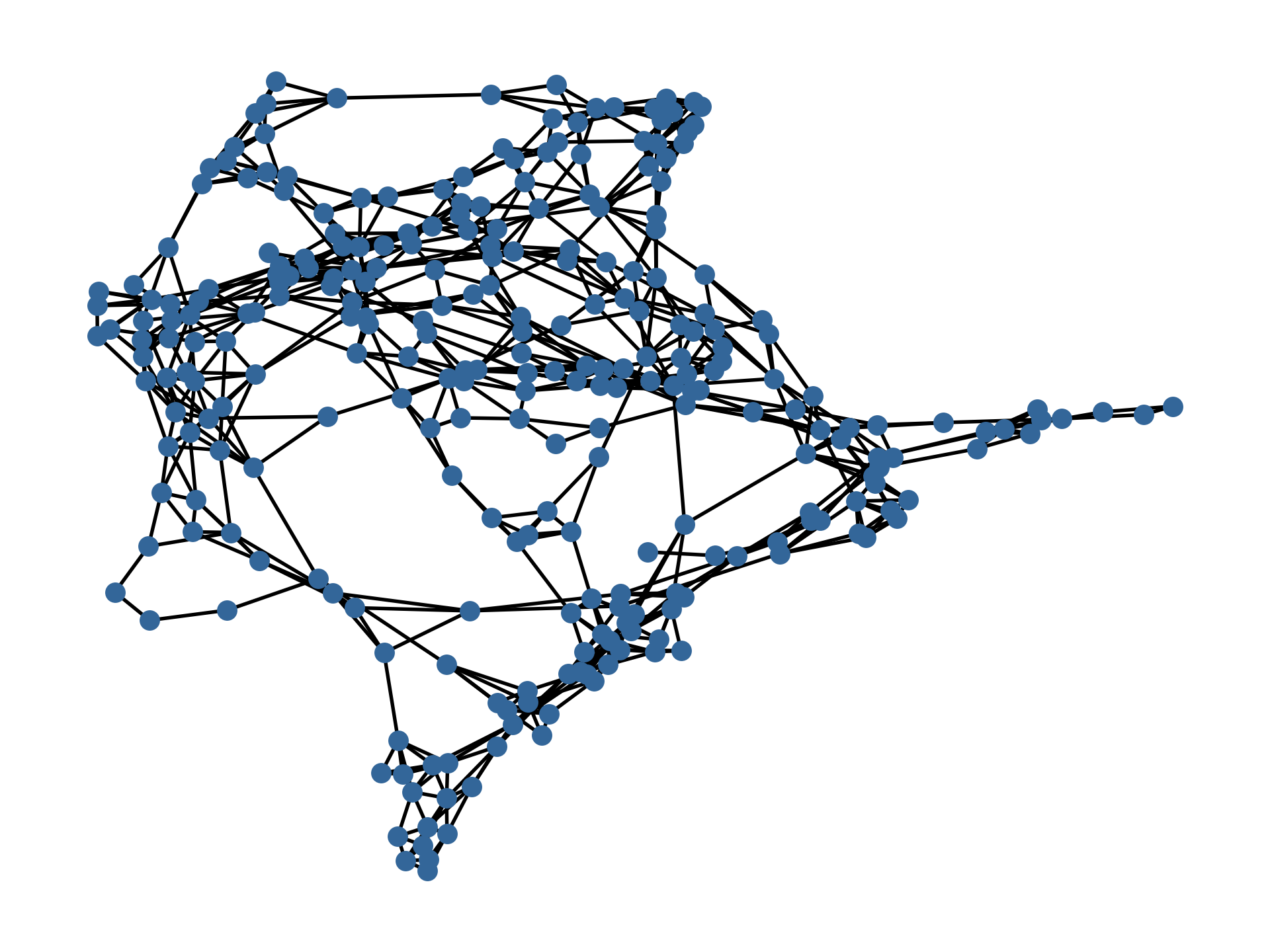}
  \\ \hrule
    \adjincludegraphics[width=.32\textwidth,
		trim={0 {0\height} 0 {.66\height}},clip]{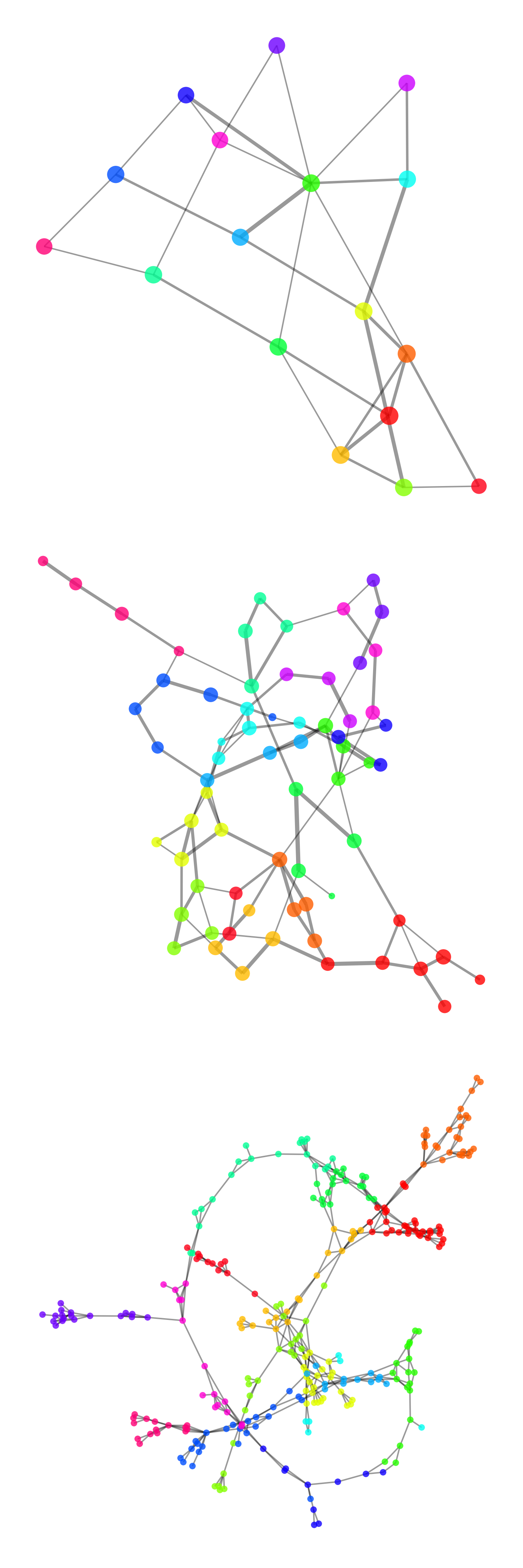}
	\adjincludegraphics[width=.32\textwidth,
		trim={0 {0\height} 0 {.66\height}},clip]{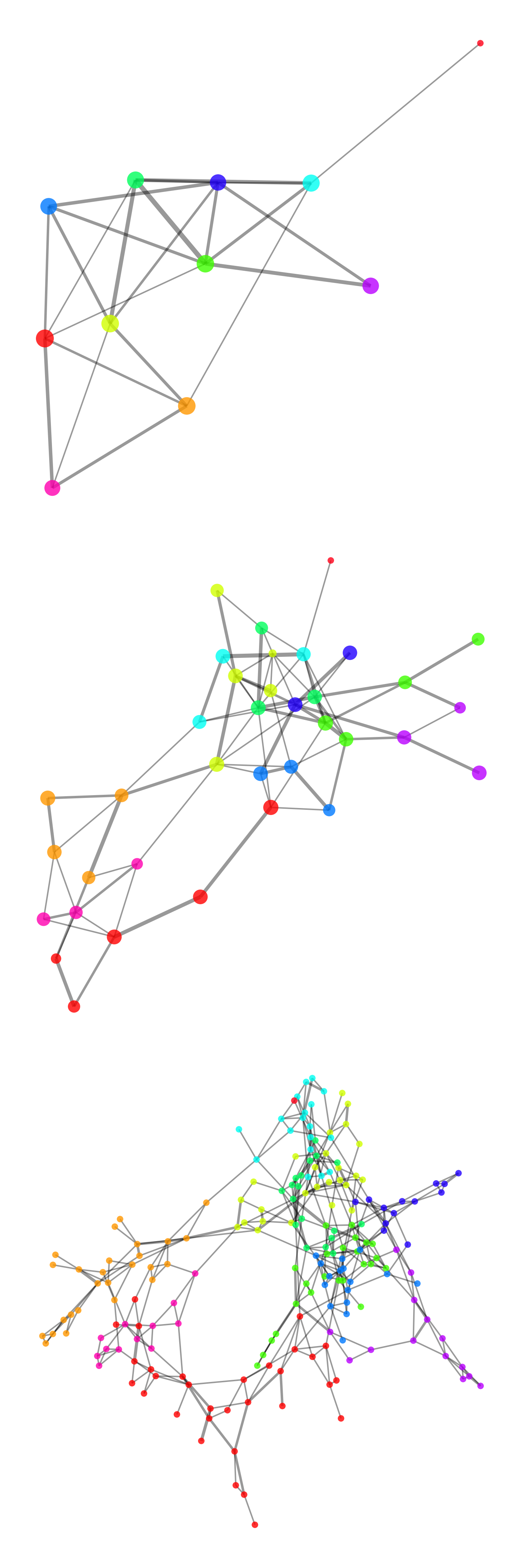}
	\adjincludegraphics[width=.32\textwidth,
		trim={0 {0\height} 0 {.66\height}},clip]{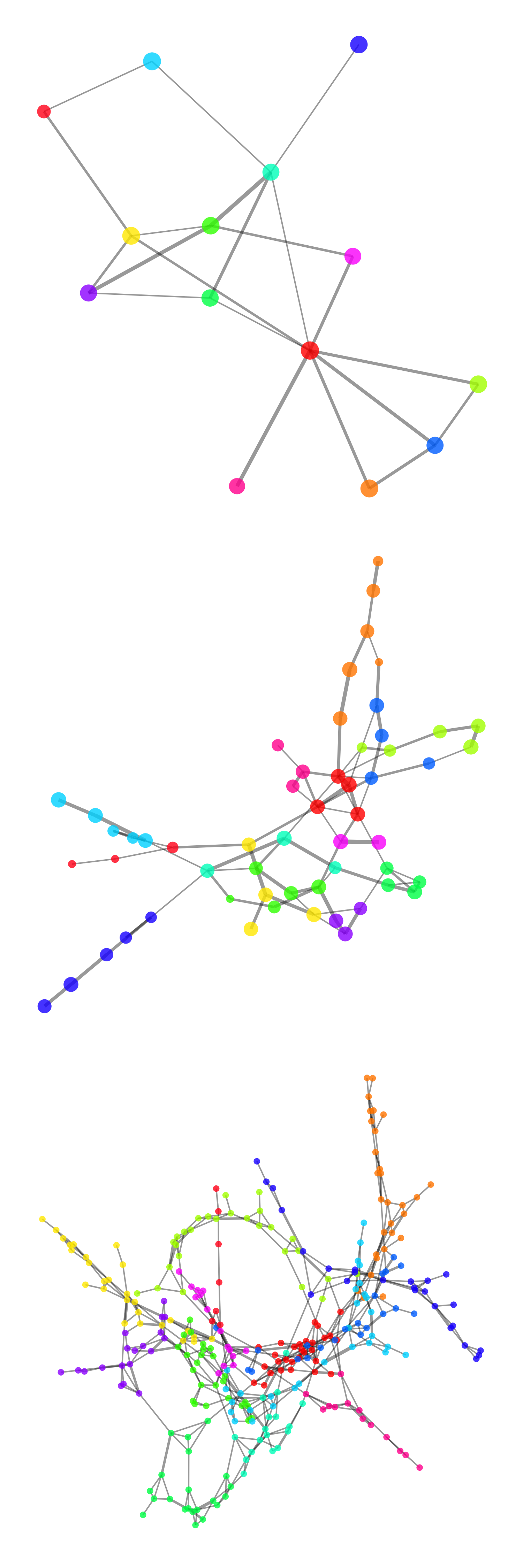}
    \end{minipage} &\hfill
    \begin{minipage}{.28\linewidth}
    \begin{center} PPG \end{center}
    \adjincludegraphics[width=.32\textwidth,
		trim={0 {0\height} 0 {.5\height}},clip]{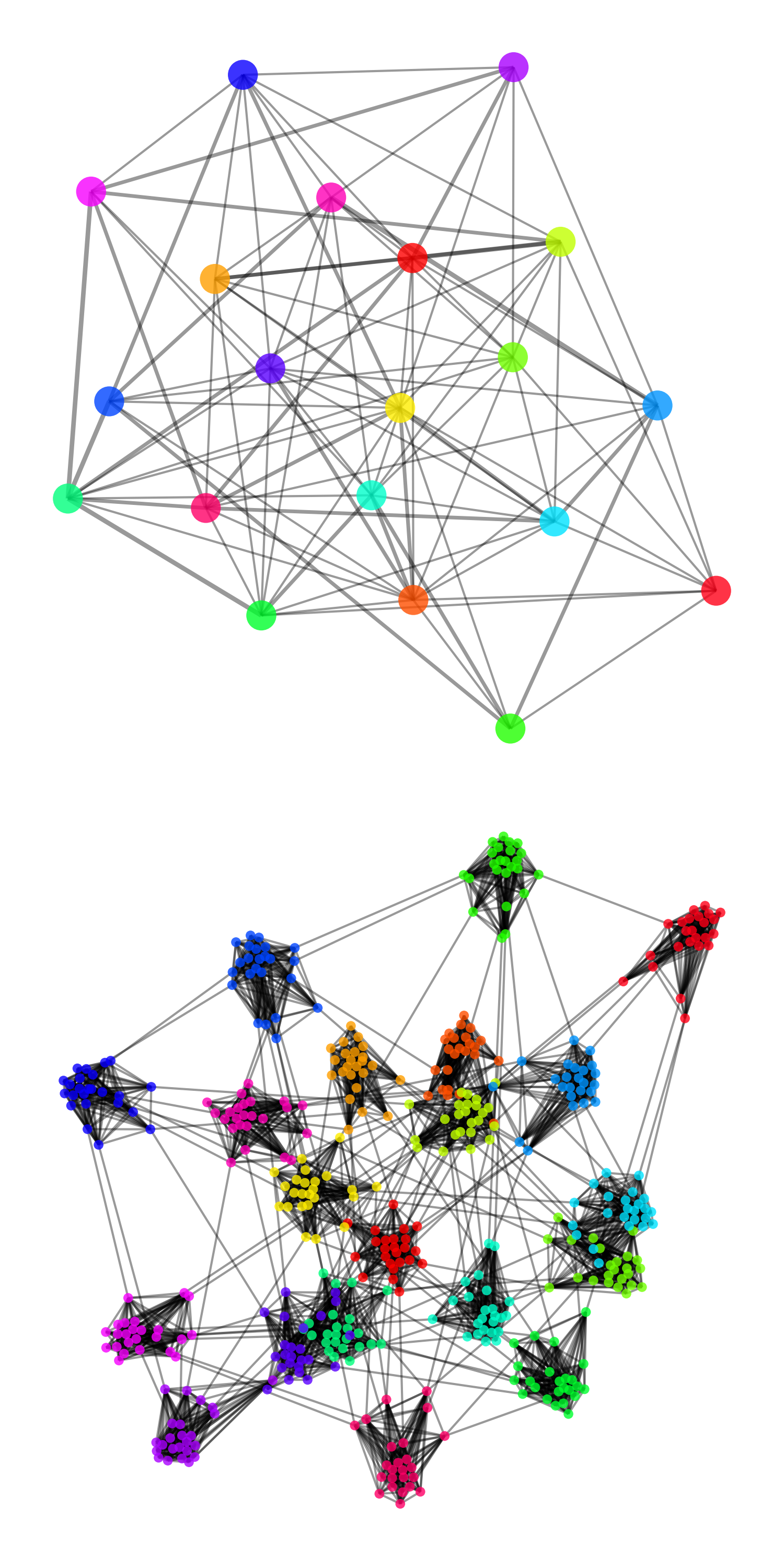}
    \adjincludegraphics[width=.32\textwidth,
		trim={0 {0\height} 0 {.5\height}},clip]{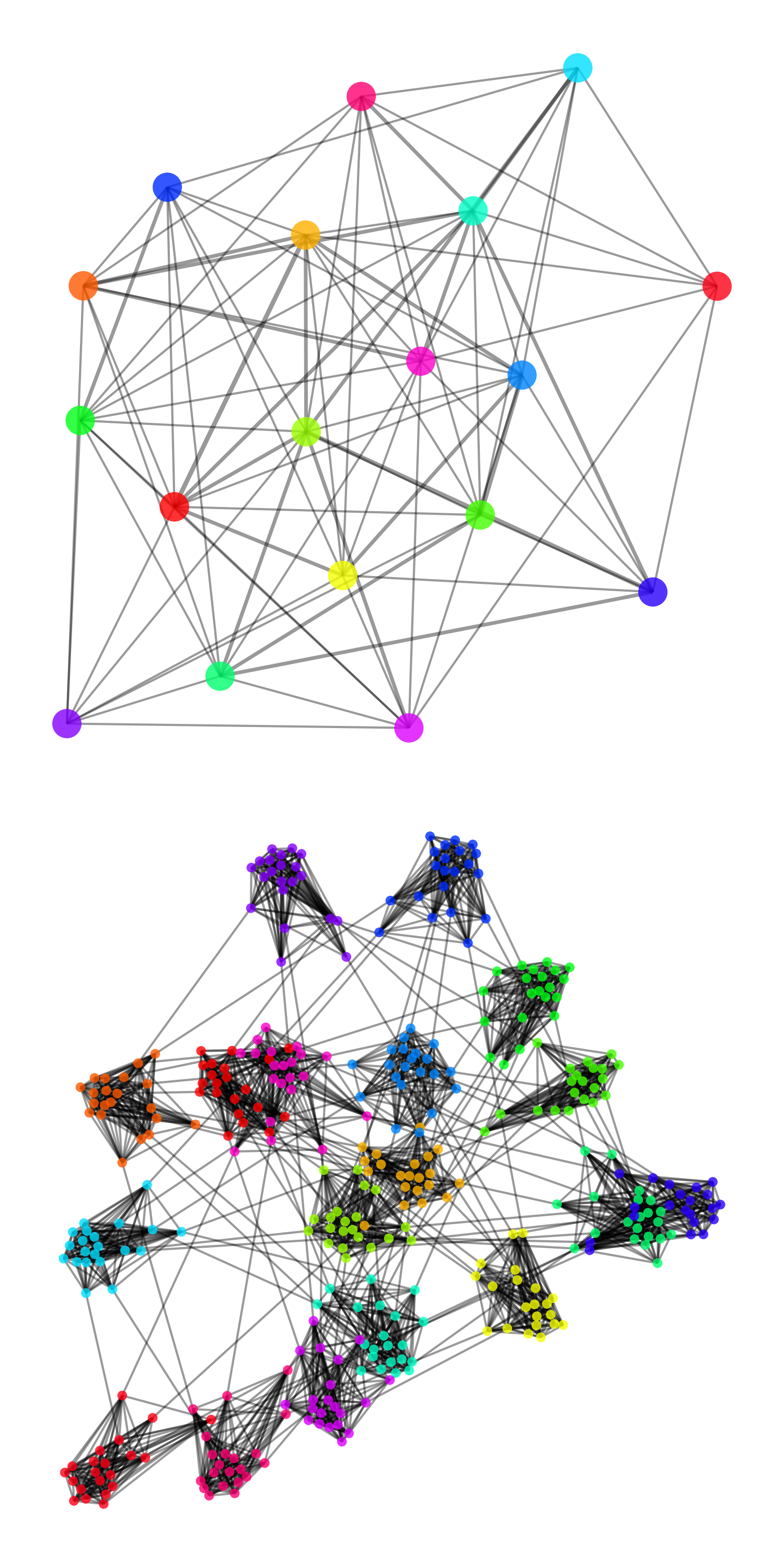}
    \adjincludegraphics[width=.32\textwidth,
		trim={0 {0\height} 0 {.5\height}},clip]{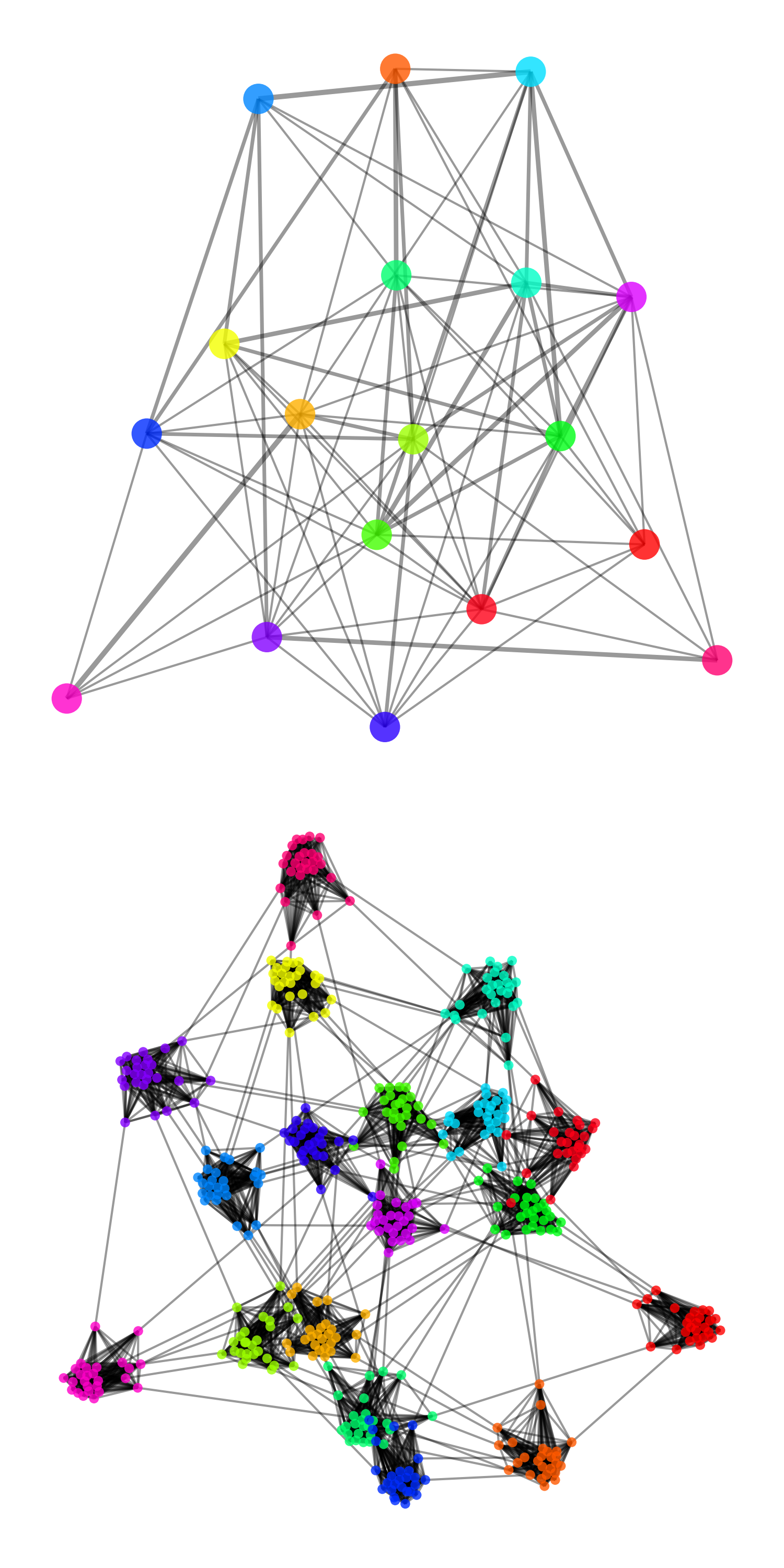}
    \\ \hrule
    \adjincludegraphics[width=.32\textwidth,
		trim={0 {0\height} 0 {0\height}},clip]{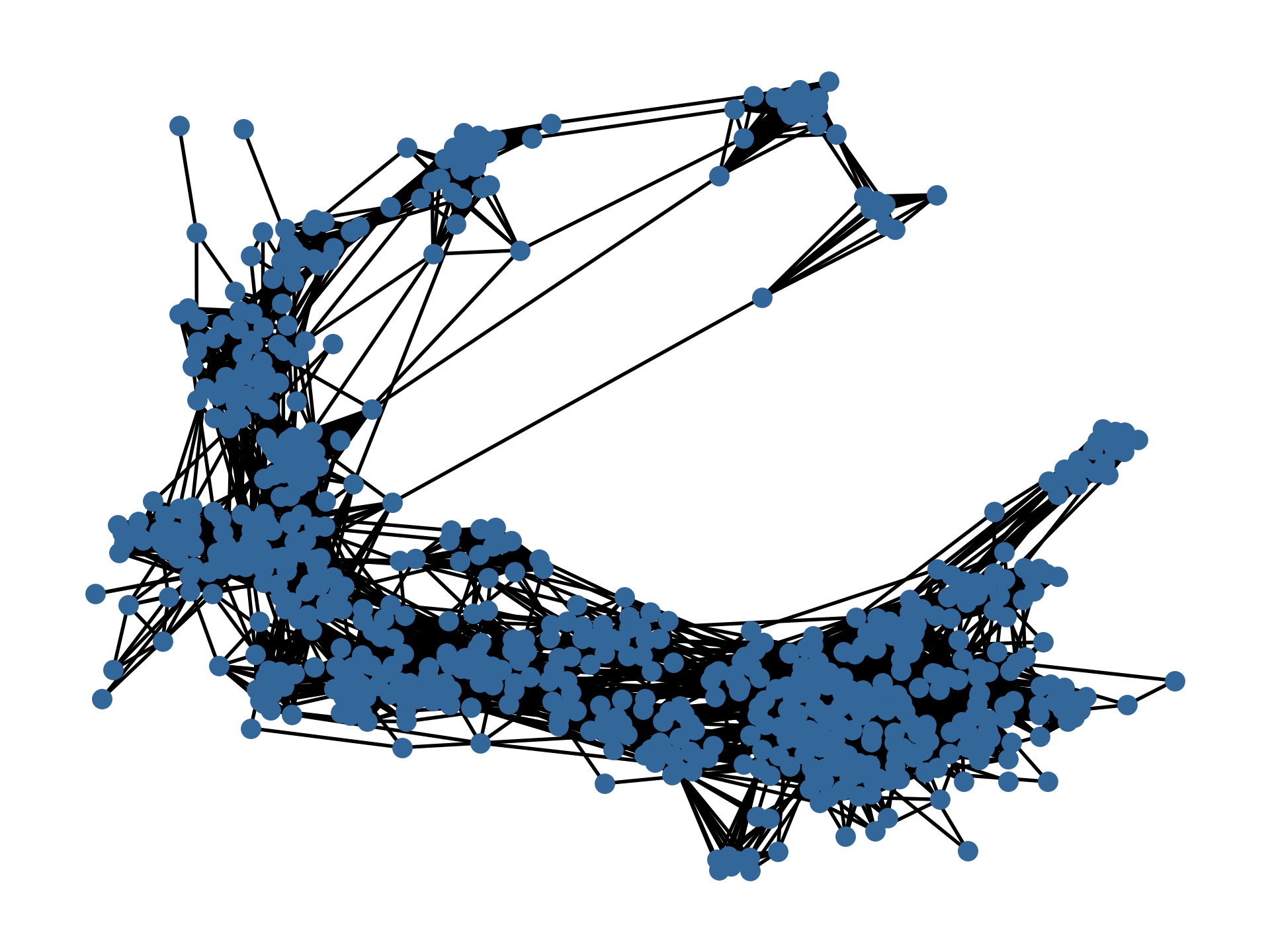}
    \adjincludegraphics[width=.32\textwidth,
		trim={0 {0\height} 0 {0\height}},clip]{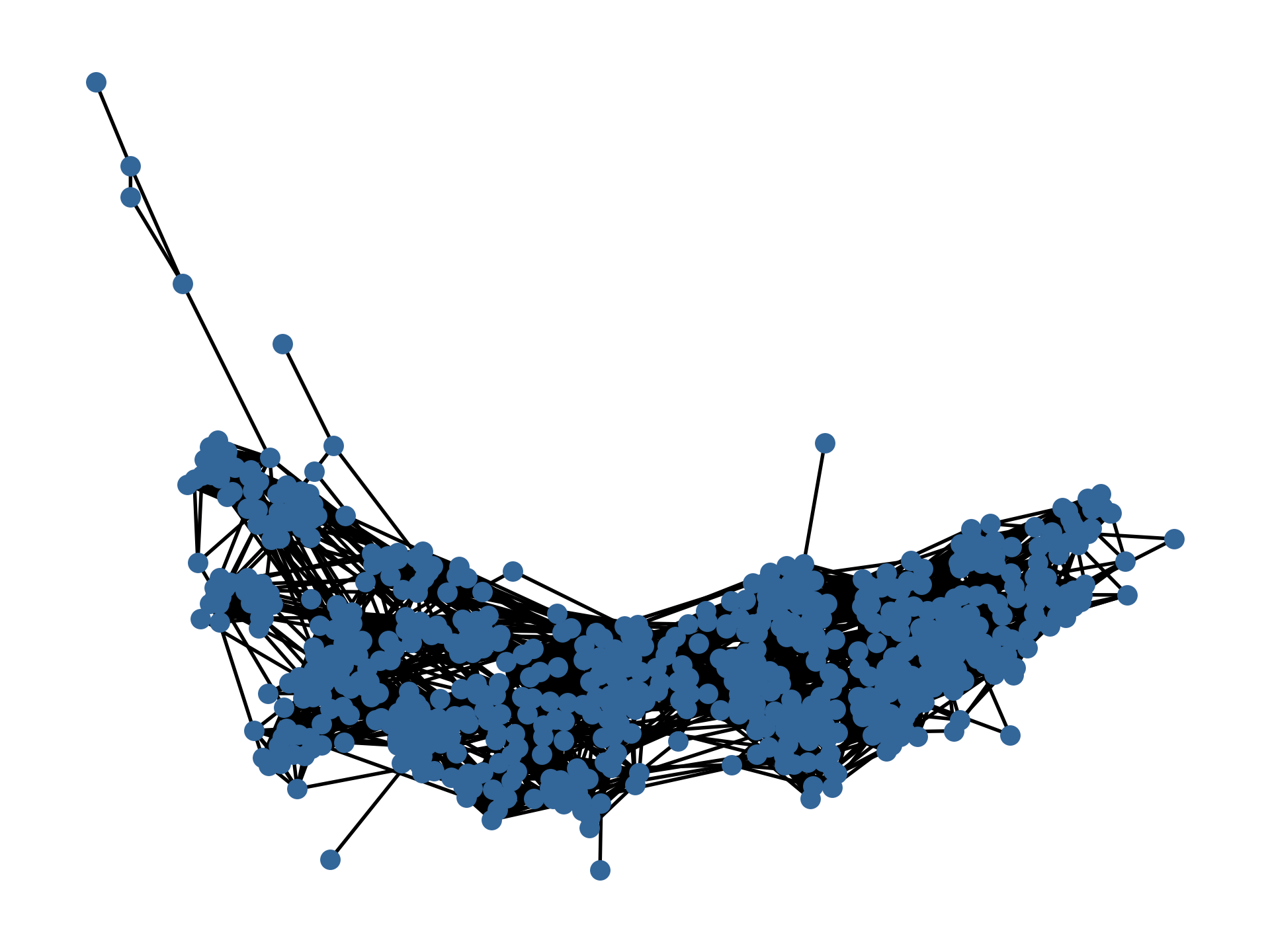}
    \adjincludegraphics[width=.32\textwidth,
		trim={0 {0\height} 0 {0\height}},clip]{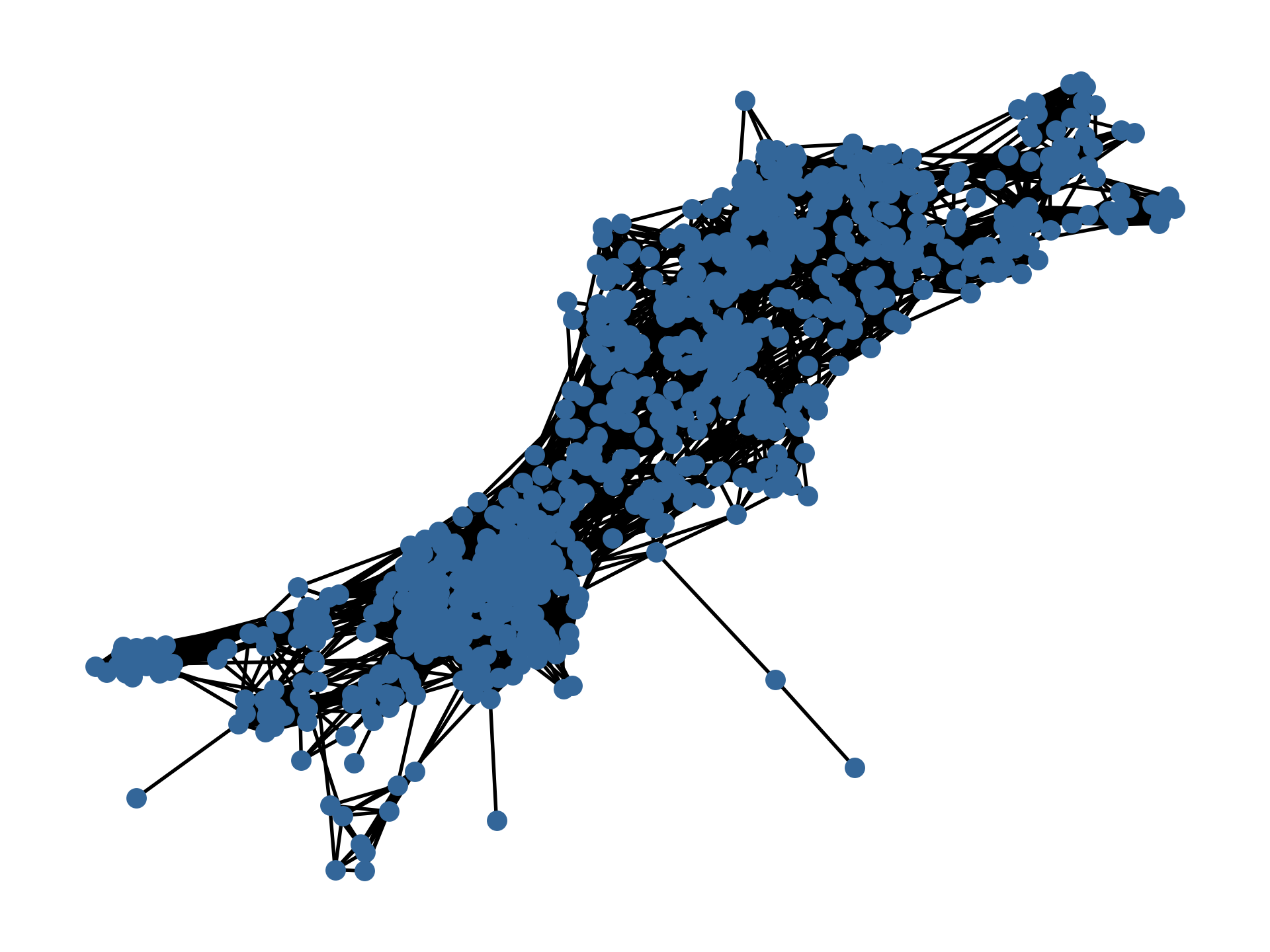}
  \\ \hrule
    \adjincludegraphics[width=.32\textwidth,
		trim={0 {0\height} 0 {0\height}},clip]{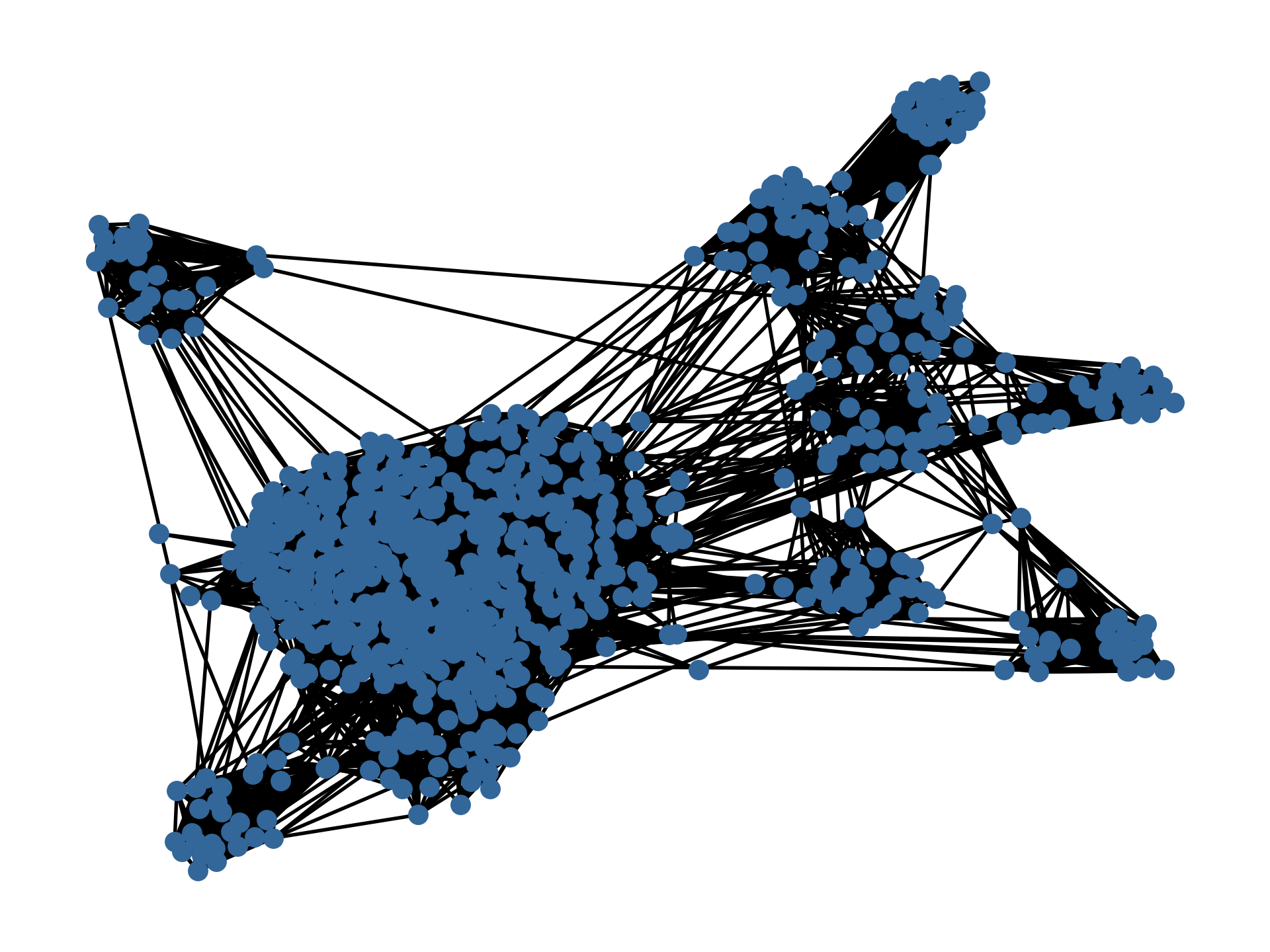}
    \adjincludegraphics[width=.32\textwidth,
		trim={0 {0\height} 0 {0\height}},clip]{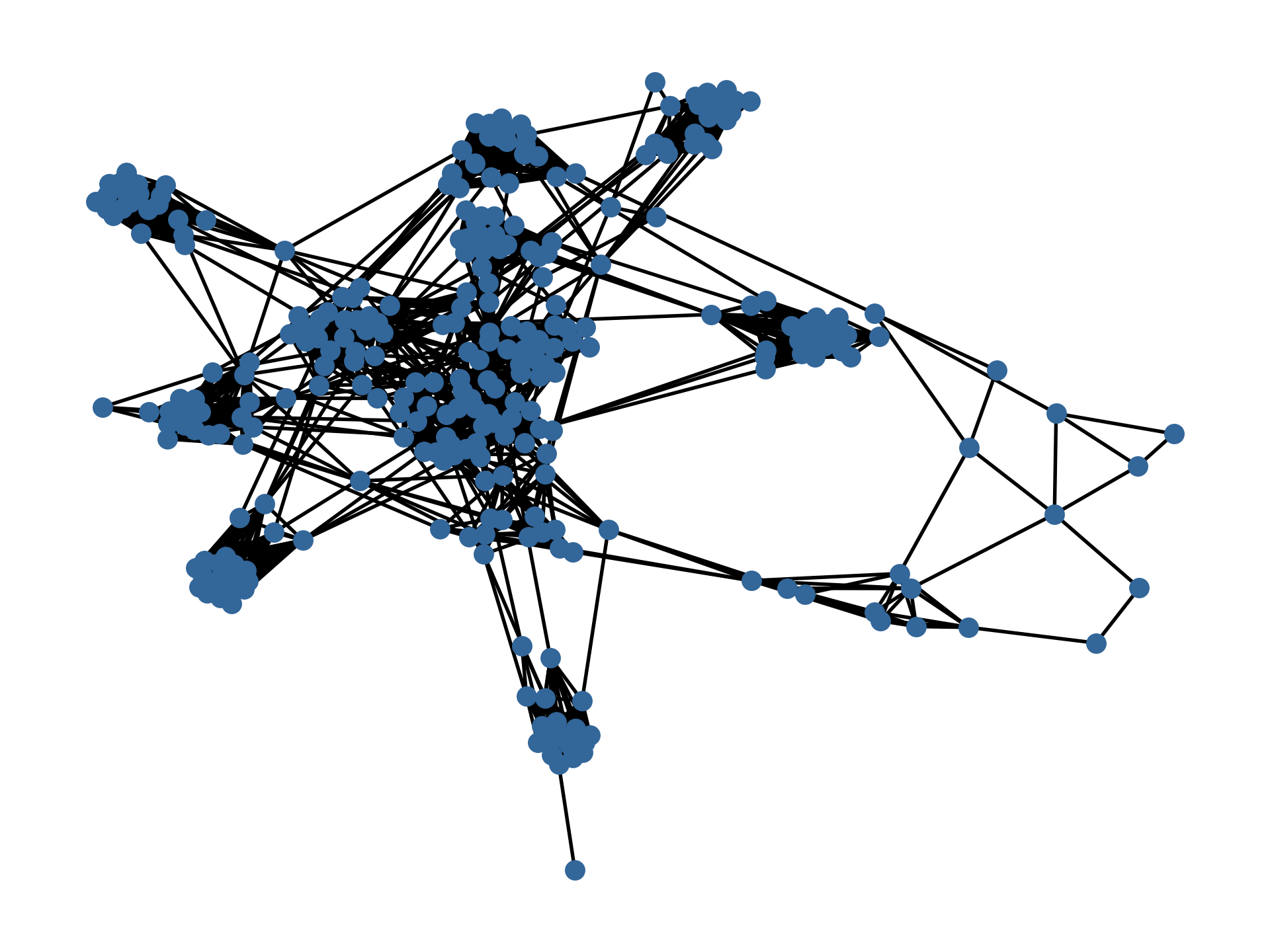}
    \adjincludegraphics[width=.32\textwidth,
		trim={0 {0\height} 0 {0\height}},clip]{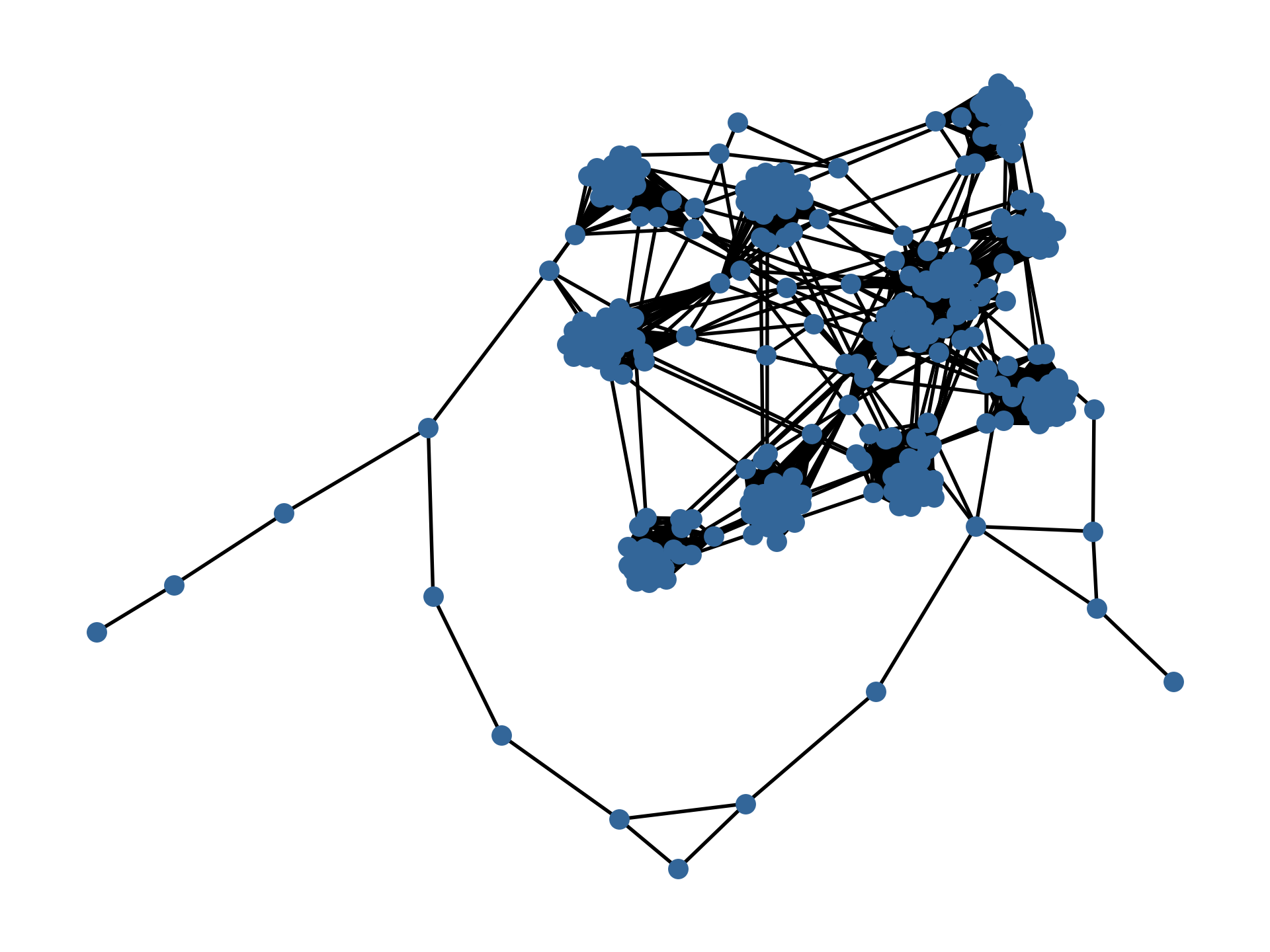}
  \\ \hrule
    \adjincludegraphics[width=.32\textwidth,
		trim={0 {0\height} 0 {.5\height}},clip]{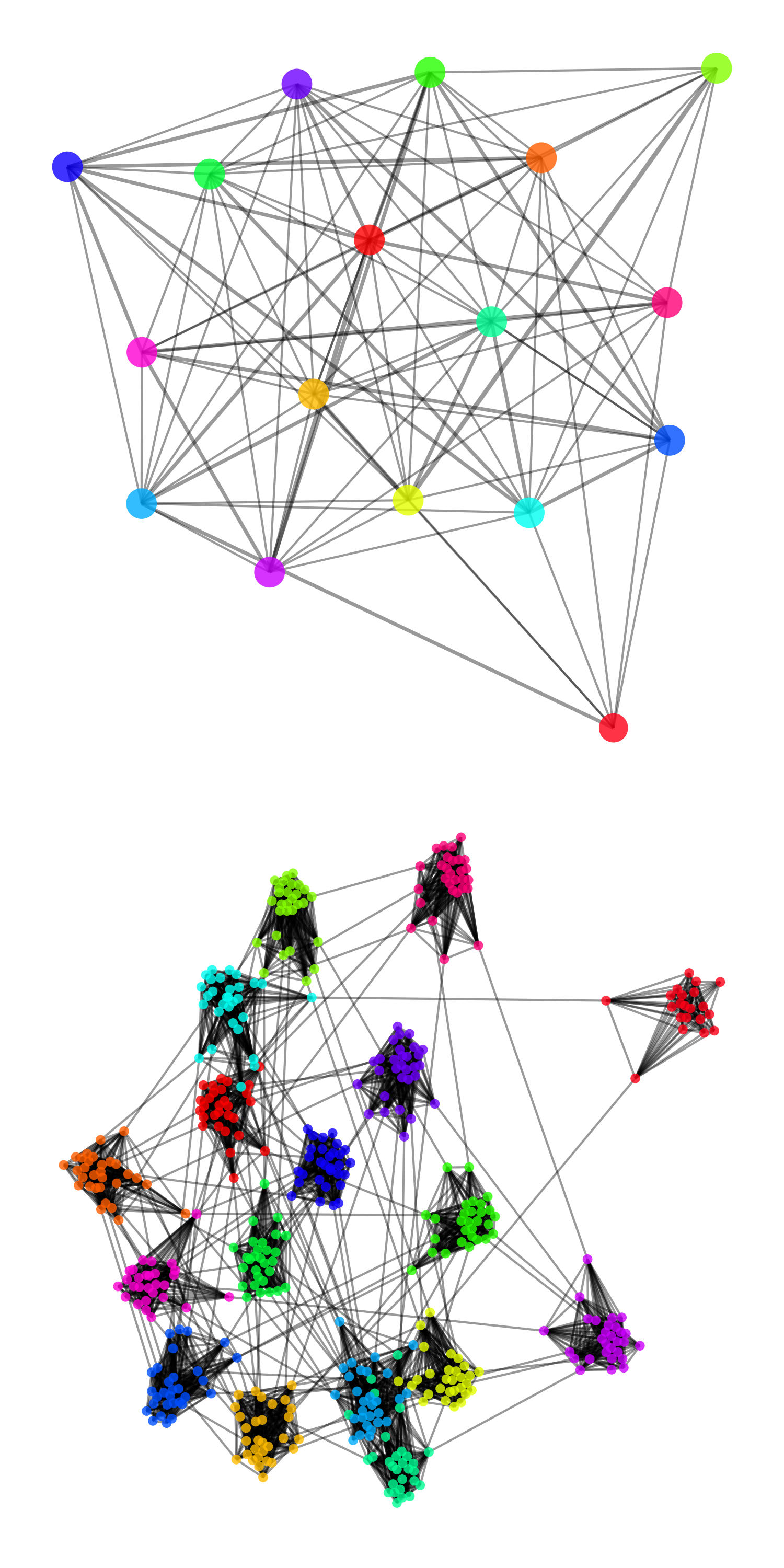}
	\adjincludegraphics[width=.32\textwidth,
		trim={0 {0\height} 0 {.5\height}},clip]{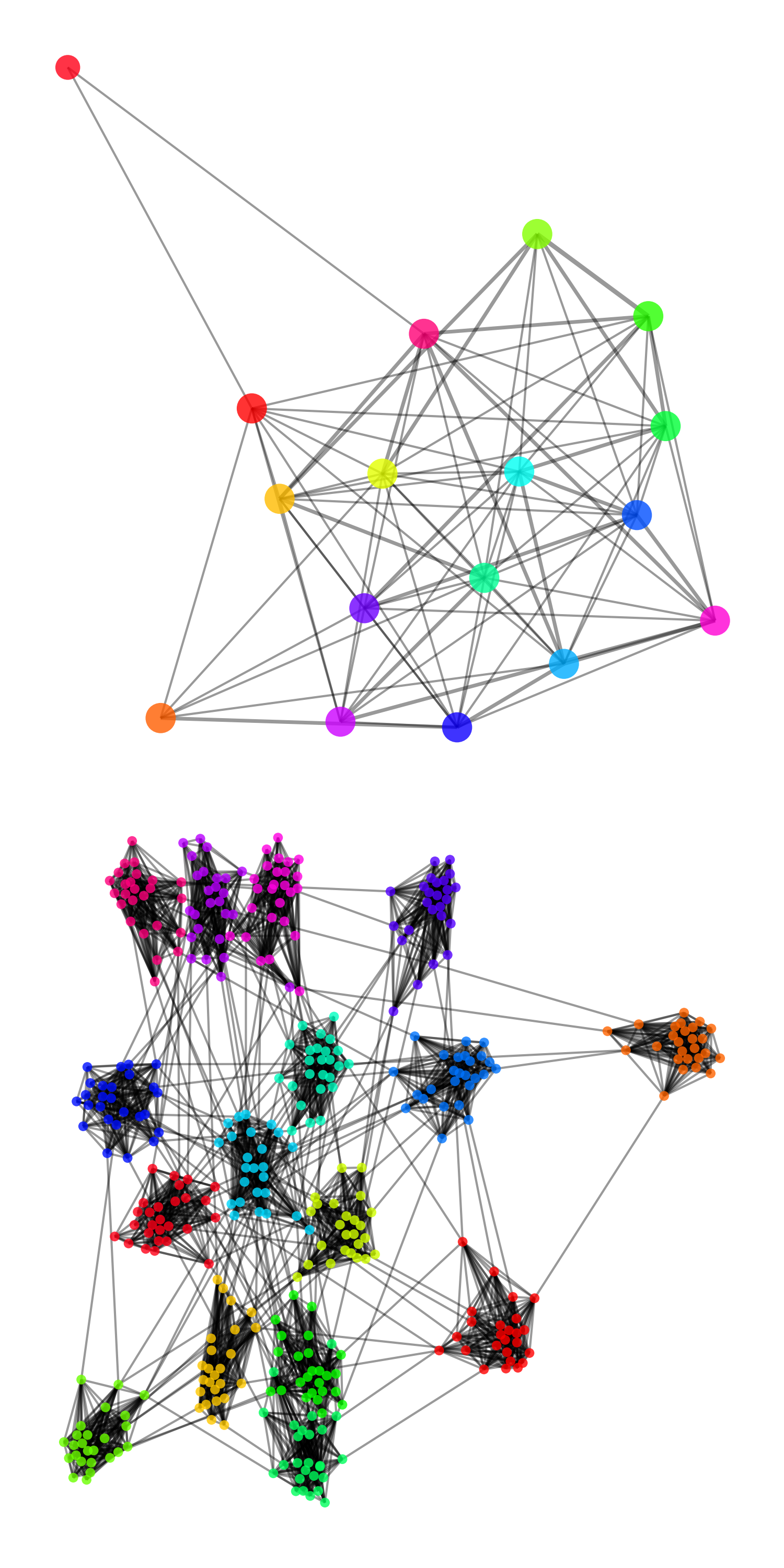}
	\adjincludegraphics[width=.32\textwidth,
		trim={0 {0\height} 0 {.50\height}},clip]{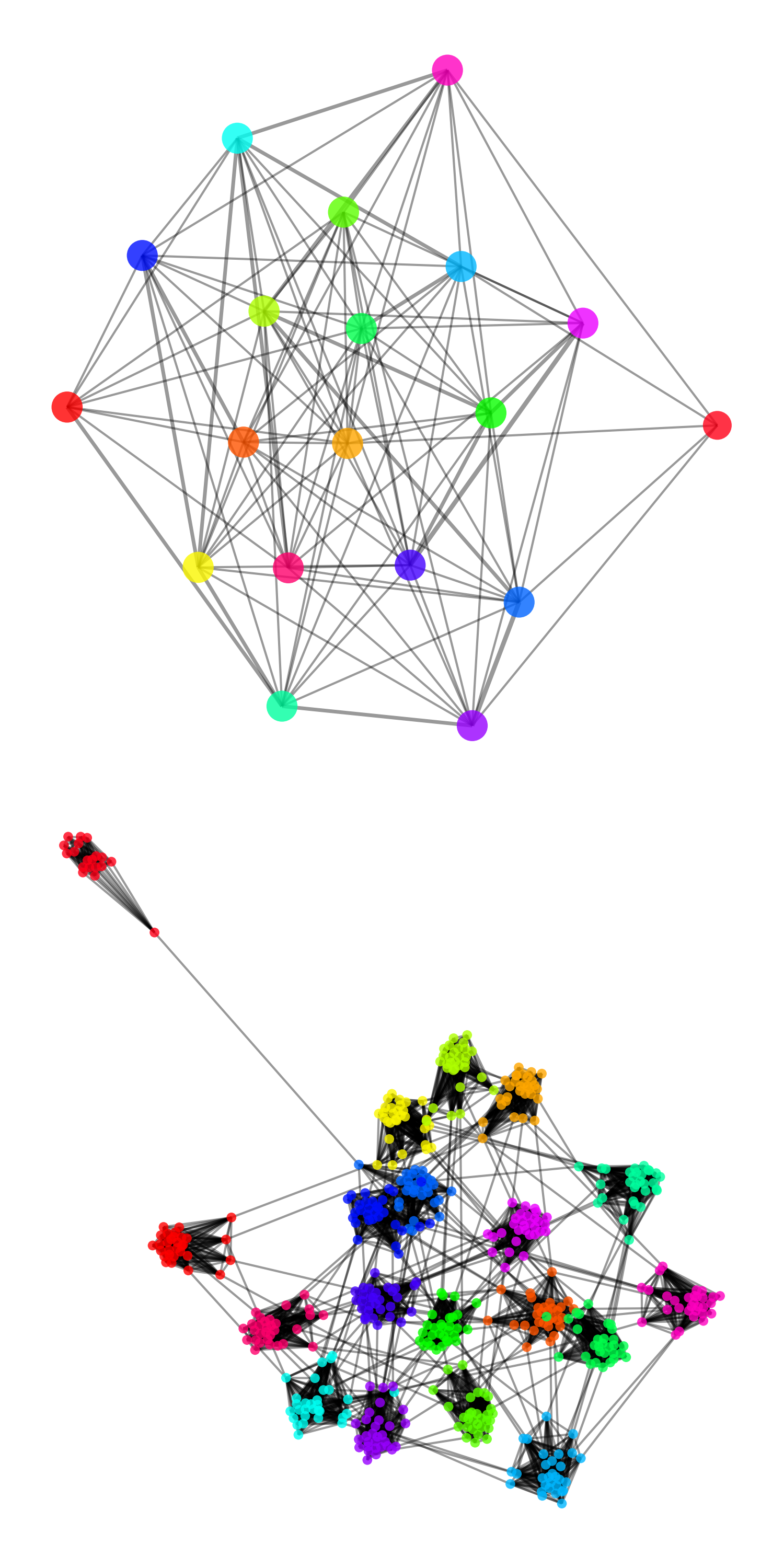}
    \end{minipage} \hfill
    \end{tabular}
\end{center}
\caption{Sample graphs generated by different models are compared to training samples at the top.
Communities are distinguished with different colors in training and MRG samples. 
\label{fig:sample_graphs}
}
\end{figure*}

\textbf{Experimental setup:}
To provide a fair comparison, we closely follow the experimental setup of \citet{you2018graphrnn} and \citet{liao2019GRAN}.
We compared the proposed model against the baseline methods including Erdos-Renyi \citep{erdos1960evolution}, GraphVAE \citep{simonovsky2018graphvae}, GraphRNN \& GraphRNN-S \citep{you2018graphrnn}, and GRAN \cite{liao2019GRAN}.
The results of  the baselines are extracted from \cite{liao2019GRAN} for the real-world graphs while we retrained GRAN for synthetic datasets.
The neural network based methods have the following structures.
GraphVAE model used a 3-layer GCN encoder and an MLP decoder with 2 hidden layers where all hidden dimensions are set to 128 for all experiments.
For GraphRNN and GraphRNN-S, the best settings reported in the original paper were used.
GRAN enjoyed 7 layers of GNNs with one round of message passing.
Hidden dimensions are set to 128 for [Ego, RCG], 256 for Point Cloud and 512 for [Protein, PPG] for GRAN,
while we used smaller hidden dimensions of 64 for [Ego, RCG, Point Cloud, Protein] and 128 for PPG.

We tested our proposed multi-resolution model (MRG) model with two variants:
1)  the model that uses mixture of multinomial distribution \eq{eq:mix_mn_pg} to describe the output distribution for all levels is simply denoted by \textbf{MRG},
2) the model that replace the output distribution of the leaf level with mixture of Bernoulli distribution is indicated by \textbf{MRG-B}.
To obtain node and edge representation, each level has its own GNN and output models, which are indexed by the level number of our model definition in section \ref{sec:HG_gen}.
We use the same GNN architecture as GRAN, with 7 layers of GNNs with one round of message passing, but we choose smaller hidden dimensions, setting it to 64 for [Ego, RCG, Point Cloud], and 128 for [Protein, PPG].
For both GRAN and MRG, the number of mixtures is set $K=20$ and block size and stride are both set to 1.
In general, MRG models uses less parameters compared to GRAN. The comparison of total number of parameters of MRG and GRAN are listed in appendix \ref{apdx:model_arch}.
MRG models are training by the Adam optimizer \cite{kingma2014adam} with learning rate of 5e-4.

For evaluation of the graph generative models, we follow the approach in \citep{liu2019graph, liao2019GRAN} which compare the following distributions of 4 different graph statistics between ground truth and generated graphs: (1) degree distributions, (2) clustering coefficient distributions, (3) the number of occurrence of all orbits with 4 nodes, and (4) the spectra of the graphs by computing the eigenvalues of the normalized graph Laplacian.
The first 3 metrics characterize local graph statistics while the spectra represents global structure.
After computing these statistics, the maximum mean discrepancy {MMD} score is computed over these statistics.
MMD score in \citep{liu2019graph} depends on Gaussian kernels with the first Wassertein distance, (the earth mover's distance (EMD)).
However, evaluating this kernel is computationally expensive for moderately large graphs, so we follow \citet{liao2019GRAN} in using total variation (TV) distance as an
alternative measure which is very faster while still consistent with EMD.
Most recently, \citet{o2021evaluationMetric} suggested using other efficient kernels such as an RBF kernel, or a Laplacian
kernel, or a linear kernel.
Also, \citet{thompson2022OnEvaluationMetric} proposed
new evaluation metrics for comparing graph sets by leveraging a random-GNN where GNNs are employed to extract meaningful graph representations.
Here we choose to comply with the experimental setup and evaluation metrics by GRAN.

The performance of the proposed graph generative models, evaluated using the maximum mean discrepancy (MMD) metric, are reported in Table \ref{table:res}.
Additionally, samples of the generated graphs are presented in Figure \ref{fig:sample_graphs}.
The results indicate that the proposed models outperform the existing graph generative models in most cases while it is on par with the best baseline in remaining cases.
This performance gap is particularly noticeable when the graph datasets has community structures.
These findings suggest that the proposed models are effective at generating graphs, particularly those with community structures, and demonstrate the potential of the proposed models in a variety of applications.
More graph samples, including their hierarchical structures, generated by the MRG models are presented in appendix \ref{apdx:modelANDsamples}.

\subsection{Ablation studies}\label{subsec:abl}
In this section, two ablation studies were conducted to evaluate more compact forms of the MRG model.

\begin{table}[t]
\begin{center}
\caption{ \captionsize
Comparison of models with different number of levels and shared model parameters across the levels (MRG-B shared).}
\label{table:res_abl}
\vskip -5pt
\begin{minipage}{1.\linewidth}
    \centering
    \begin{adjustbox}{width=.95\textwidth,}
        \begin{tabular}{
l |cccc}
                    & \multicolumn{4}{c}{\textbf{Ego}} \\
                     & Deg.     & Clus.    & Orbit  & Spec.   \\ \toprule
\textbf{MRG-B 3-level}       & 4.1$e^{-3}$  & 6.2$e^{-2}$ & 1.8$e^{-2}$  & 1.42$e^{-2}$ \\
\textbf{MRG-B 2-level}         & 4.73$e^{-3}$  & 5.43$e^{-2}$ & 1.41$e^{-2}$  & 1.9$e^{-2}$    \\
\textbf{MRG-B shared}         & 1.87$e^{-2}$  & 3.68$e^{-1}$ & 3.20$e^{-2}$  & 3.16$e^{-2}$    \\
\hline
\end{tabular}
    \end{adjustbox}
\end{minipage}
\end{center}
\end{table}

\begin{table}[t]
\begin{center}
\caption{ \captionsize
Ablation study on node ordering. 
Baseline MRG used the
BFS ordering and baseline GRAN used DFS ordering.
$\pi_1$ and $\pi_2$ are default and random node ordering, respectively. }
\label{table:res_abl_ordering}
\vskip -5pt
\begin{minipage}{1.\linewidth}
    \centering
    \begin{adjustbox}{width=.9\textwidth,}
        \begin{tabular}{
l |cccc}
                    & \multicolumn{4}{c}{\textbf{Protein}} \\
                     & Deg.     & Clus.    & Orbit  & Spec.   \\ \toprule
\textbf{GRAN}                 & {1.98$e^{-3}$} & {4.86$e^{-2}$ } & 1.3$e^{-1}$   & {5.13$e^{-3}$} \\
\textbf{GRAN ($\pi_1$)}         & 9.2$e^{-2}$  & 0.12 & 0.74 & 3.4$e^{-2}$   \\
\textbf{GRAN ($\pi_2$)}         & 0.70  & 1.04 & 1.40  & 0.64    \\
\hline \rule{0pt}{3pt}
\textbf{MRG-B}                 & 5.1$e^{-3}$  & 6.27$e^{-2}$ & 1.08$e^{-1}$ & 8.0$e^{-3}$  \\
\textbf{MRG-B ($\pi_1$)}         & 1.28$e^{-2}$  & 1.03$e^{-1}$ & 2.85$e^{-2}$  & 1.29$e^{-2}$    \\ %
\textbf{MRG-B ($\pi_2$)}         & 9.53$e^{-3}$  & 7.46$e^{-2}$ & 2.61$e^{-1}$  & 1.34$e^{-2}$    \\ %
\hline
\end{tabular}

    \end{adjustbox}
\end{minipage}
\end{center}
\end{table}
The first study evaluated the performance of MRG with fewer hierarchical levels by splicing out the middle level of the Ego dataset, resulting in hierarchical graphs (HGs) with only 2 levels after the root, \ie $L=2$. The results, presented in Table \ref{table:res_abl}, show that the generation quality of the models drops slightly when the number of levels is decreased, indicating that having more hierarchical levels improves the expressiveness of the model.

Moreover, we train the MRG with shared model parameters across levels such that all levels use similar GNN and output models.
The performance comparisons in Table \ref{table:res_abl} show that using individual models for each level offers better results. This can can be explained by the fact that graph at different levels exhibits different characteristics such as graph sparsity that may require tailored models for optimal performance.

\section{Conclusion}\label{sec:concl}

We proposed a novel data-drive generative model for generic hierarchical graphs.
This model does not rely on domain-specific priors and can be used widely.
Our method also supports maximally parallelized implementations
insofar as the graph is amenable to balanced recursive tree decomposition.
We demonstrated the effectiveness and efficiency of our method on 2 synthetic and 3 real datasets.
While the Louvain algorithm we depend on for community detection is rule-based, still the proposed method is proven to be effective.

For future work, developing a fully end-to-end algorithm for encoding and decoding with joint learning of community structures, instead of depending on an external algorithm for community detection, will be both challenging and desirable.
Moreover, both for the current method using various community-detection algorithms and for the future end-to-end solution, validation on datasets that are orders of magnitude bigger than what we used in this work to introduce the new method will be an informative and worthy undertaking.

\subsubsection*{Acknowledgments}
We would like to thank Fatemeh Fani Sani for preparing the schematic figures. %

\bibliography{iclr2023_conference}
\bibliographystyle{iclr2023_conference}

\appendix
\section{Appendix}

\begin{appendices}

\onecolumn 
\section{Proof of Theorem \ref{thm:mn2bnmn}} \label{apdx:proof_mn2bnmn}

For a random counting vector $\rvw \in \ZP^E$ with multinomial distribution $\text{Mu}(w, \vtheta)$, %
let's split it into $M$ disjoint groups $\rvw=[\rvu_1, ~ ..., \rvu_M ]$ where $\rvu_m \in \ZP^{E_m} ~,~ \sum_{m=1}^M {E_m} = E $, 
and also split the probability vector as $\vtheta=[\vtheta_1, ~ ..., \vtheta_M ]$.
Additionally, let's define sum of all weights in $m$-th group by a random variable $\rv_m := \sum_{e=1}^{E_m} \ru_{m,e}$. 

\begin{lemma} \label{lemma:groupsMN}
Sum of the weights in the groups, $\rvu_m \in \ZP^{E_m} ~,~ \sum_{m=1}^M {E_m} = E $ has multinomial distribution:
\begin{align} \label{eq:proof_lemma_groupsMN}
p(\{\rv_1, ..., \rv_M\}) &=  \text{Mu}(w, ~[\alpha_1, ..., \alpha_M]) \nonumber\\
\text{where: }\alpha_m &= \sum \vtheta_m[i].
\end{align}
In the other words, the multinomial distribution is preserved when its counting variables are combined \cite{siegrist2017probability}.
\end{lemma}

\begin{lemma} \label{lemma:ConditionalMN}
Given the sum of counting variables in the groups, the groups are independent and each of them has  multinomial distribution: 
\begin{align} \label{eq:lemma_mn2bnmn}
p(\rvw =[\rvu_1, ~ ..., \rvu_M ]| \{\rv_1, ..., \rv_M\}) &= 
\prod_{m=1}^{M} \text{Mu}( \rv_{m}, ~ \vlambda_{{m}}) \\
\text{where: }    \vlambda_{{m}} & = \frac{\vtheta_m}{\1^{T}~ \vtheta_m} 
\nonumber
\end{align}
Here, probability vector (parameter) $\vlambda_{{m}}$ is the normalized multinomial probabilities of the counting variables in the $m$-th group.   

\begin{proof}
\begin{align} \label{eq:proof_lemma_mn2bnmn}
p(\rvw | \{\rv_1, ..., \rv_M\}) 
&= \frac{p(\rvw)}{p(\{\rv_1, ..., \rv_M\})} I(\rv_1= \1^{T}~\rvu_1, ~ ..., \rv_M=\1^{T}~\rvu_M )  \nonumber\\
&= \frac{
\frac{w!}{\prod_{i=1}^{E} \rw_i ! } \prod_{i=1}^{E} {\vtheta_i}^{\rw_i }
}{
\frac{w!}{\prod_{i=1}^{M} \rv_i ! } \prod_{i=1}^{M} {\alpha_i}^{\rv_i }
} I(\rv_1= \1^{T}~\rvu_1, ~ ..., \rv_M=\1^{T}~\rvu_M )  \nonumber\\
&= \frac{
\frac{w!}{\prod_{i=1}^{E} \rw_i ! } \vtheta_1^{\rw_1 } ... \vtheta_E^{\rw_E }
}{
\frac{w!}{\prod_{i=1}^{M} \rv_i ! } {(\1^{T}~ \vtheta_1})^{\rv_1 } ... ({\1^{T}~ \vtheta_M})^{\rv_M }
} \nonumber\\
&= \frac{\rv_1!}{\prod_{i=1}^{E_1} \rvu_{1,i} ! } \prod_{i=1}^{E_1} {\vlambda_{1,i}}^{\rvu_{1,i}} \times ... \times
\frac{\rv_M!}{\prod_{i=1}^{E_M} \rvu_{M,i} ! } \prod_{i=1}^{E_1} {\vlambda_{M,i}}^{\rvu_{M,i}} 
 \nonumber \\
&= \text{Mu}( \rv_{1}, ~ \vlambda_{{1}}) \times ... \times \text{Mu}( \rv_{M}, ~ \vlambda_{{M}}) \nonumber
\end{align}
\end{proof} 
\end{lemma}

\begin{theorem} 
Given the aforementioned grouping of counts variables, the multinomial distribution can be modeled as a chain of binomials and multinomials:   
\begin{align} %
    \text{Mu}(w, \vtheta=[\vtheta_1, ..., \vtheta_M ]) & = \prod_{m=1}^{M} \text{Bi}( w - \sum_{i<m}\rv_i, ~ {\eta}_{\rv_{m}}) ~
    \text{Mu}( \rv_{m}, ~ \vlambda_{{m}}),  \\
    \text{where: } \eta_{\rv_m } & = \frac{\1^{T}~ \vtheta_m}{1-\sum_{i<m} \1^{T}~ \vtheta_i}, \\   
    \vlambda_{{m}} & = \frac{\vtheta_m}{\1^{T}~ \vtheta_m} \nonumber
\end{align}
\end{theorem}  

\begin{proof} 
	Since sum of the weights of the groups, $\rv_m$, are functions of the weights in the group:
	\begin{align} \label{eq:proof_mn2bnmn1}
	p(\rvw) = p(\rvw, \{\rv_1, ..., \rv_M\}) = p(\rvw | \{\rv_1, ..., \rv_M\}) p(\{\rv_1, ..., \rv_M\}) \nonumber
	\end{align}
	According to lemma \ref{lemma:groupsMN}, sum of the weights of the groups is a multinomial and by lemma \ref{thm:mn2bn}, it can be decomposed to a sequence of binomials:
	\begin{align} %
	p(\{\rv_1, ..., \rv_M\}) 
	&=  \text{Mu}(w, ~[\alpha_1, ..., \alpha_M]) \nonumber\\
	&= \prod_{m=1}^{M} \text{Bi}( w - \sum\nolimits_{i<m}\rv_i, \hat{\eta}_m), \nonumber\\
	\text{where: }\alpha_m &= \1^{T}~ \vtheta_m, ~   \hat{\eta}_e = \frac{\alpha_{e}}{1-\sum_{i<e}\alpha_m} \nonumber
	\end{align}
	Also based on lemma \ref{lemma:ConditionalMN}, given the sum of the wights of all groups, the groups are independent and has  multinomial distribution: 
	\begin{align} %
	p(\rvw | \{\rv_1, ..., \rv_M\}) &= 
	\prod_{m=1}^{M} \text{Mu}( \rv_{m}, ~ \vlambda_{{m}}) \\
	\text{where: }    \vlambda_{{m}} & = \frac{\vtheta_m}{\1^{T}~ \vtheta_m} 
	\nonumber
	\end{align}
\end{proof}

\section{Generated samples} \label{apdx:modelANDsamples}

Generated hierarchical graphs sampled MRG models are presented in this section. 

\begin{figure}[h!] \label{fig:sample_HG2} 
\begin{center}
    \begin{tabular}{c c|c|c}
    \begin{minipage}{.02\linewidth}
    \begin{turn}{90} \hspace{1pt} MRG \hspace{50pt}  Train  \end{turn}
    \end{minipage} &\hfill
    \begin{minipage}{.31\linewidth}
    \begin{center} Ego \end{center}
    \adjincludegraphics[width=.32\linewidth,
		trim={0 {0\height} 0 {0.\height}},clip]{./FigureTable/Ego_train_1.png}
    \adjincludegraphics[width=.32\linewidth,
		trim={0 {0\height} 0 {0.\height}},clip]{./FigureTable/Ego_train_2.png}
    \adjincludegraphics[width=.32\linewidth,
		trim={0 {0\height} 0 {0.\height}},clip]{./FigureTable/Ego_train_3.png}
  \\ \hrule
    \adjincludegraphics[width=.32\textwidth,
		trim={0 {0\height} 0 {0.\height}},clip]{./FigureTable/Ego_model_1.png}
	\adjincludegraphics[width=.32\textwidth,
		trim={0 {0\height} 0 {0.\height}},clip]{./FigureTable/Ego_model_2.png}
	\adjincludegraphics[width=.32\textwidth,
		trim={0 {0\height} 0 {0.\height}},clip]{./FigureTable/Ego_model_3.png}
    \end{minipage} &\hfill
    \begin{minipage}{.31\linewidth}
    \begin{center} Protein \end{center}
    \adjincludegraphics[width=.32\textwidth,
		trim={0 {0\height} 0 {0.\height}},clip]{./FigureTable/DD_train_1.png}
    \adjincludegraphics[width=.32\textwidth,
		trim={0 {0\height} 0 {0.\height}},clip]{./FigureTable/DD_train_2.png}
    \adjincludegraphics[width=.32\textwidth,
		trim={0 {0\height} 0 {0.\height}},clip]{./FigureTable/DD_train_3.png}
  \\ \hrule
    \adjincludegraphics[width=.32\textwidth,
		trim={0 {0\height} 0 {0.\height}},clip]{./FigureTable/DD_model_1.png}
	\adjincludegraphics[width=.32\textwidth,
		trim={0 {0\height} 0 {0.\height}},clip]{./FigureTable/DD_model_2.png}
	\adjincludegraphics[width=.32\textwidth,
		trim={0 {0\height} 0 {0.\height}},clip]{./FigureTable/DD_model_3.png}
    \end{minipage}  &\hfill
    \begin{minipage}{.32\linewidth}
    \begin{center} Point Cloud \end{center}
    \adjincludegraphics[width=.32\textwidth,
		trim={0 {0\height} 0 {0.\height}},clip]{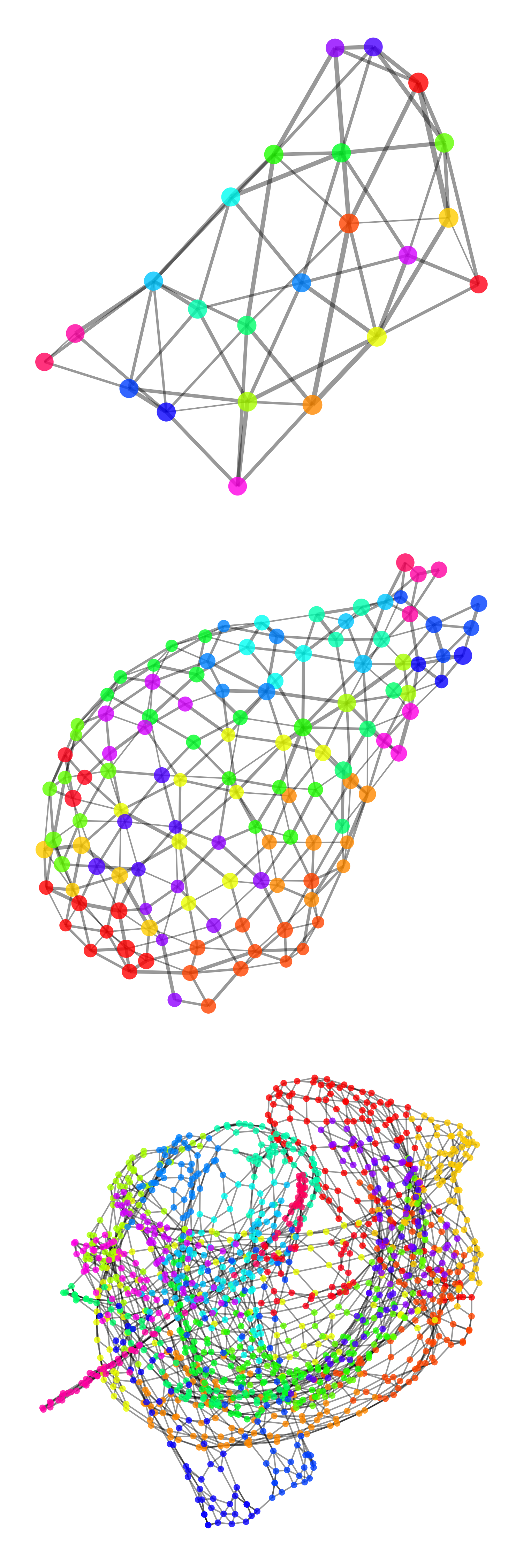}
    \adjincludegraphics[width=.32\textwidth,
		trim={0 {0\height} 0 {0.\height}},clip]{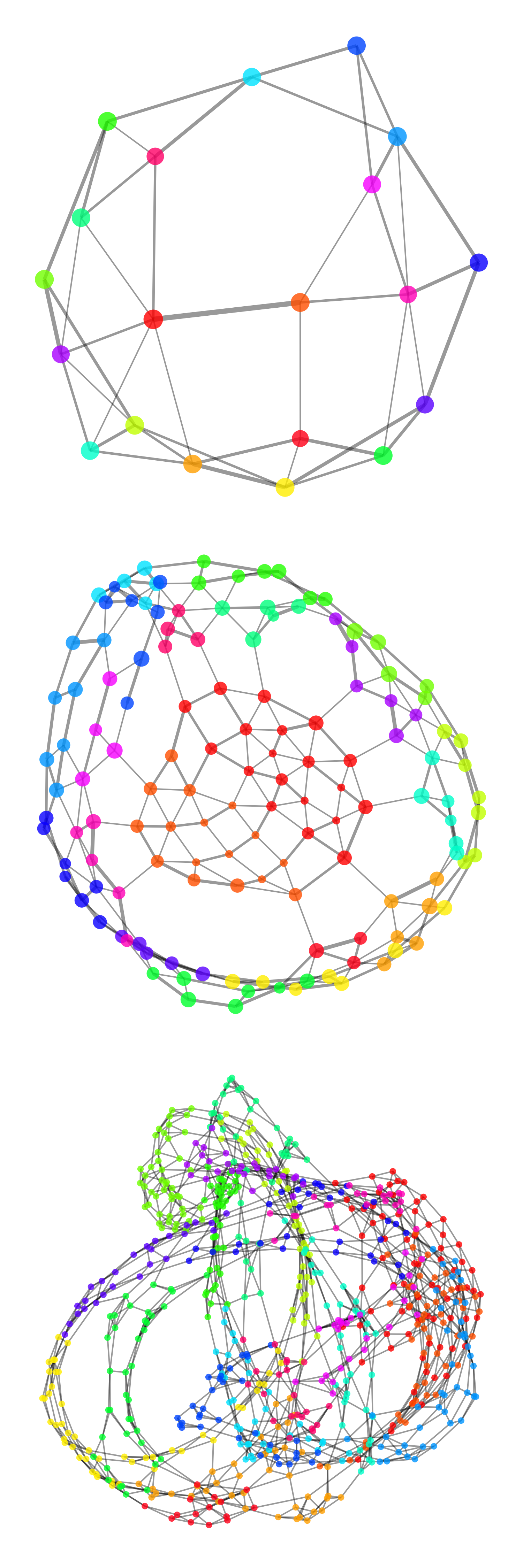}
    \adjincludegraphics[width=.32\textwidth,
		trim={0 {0\height} 0 {0.\height}},clip]{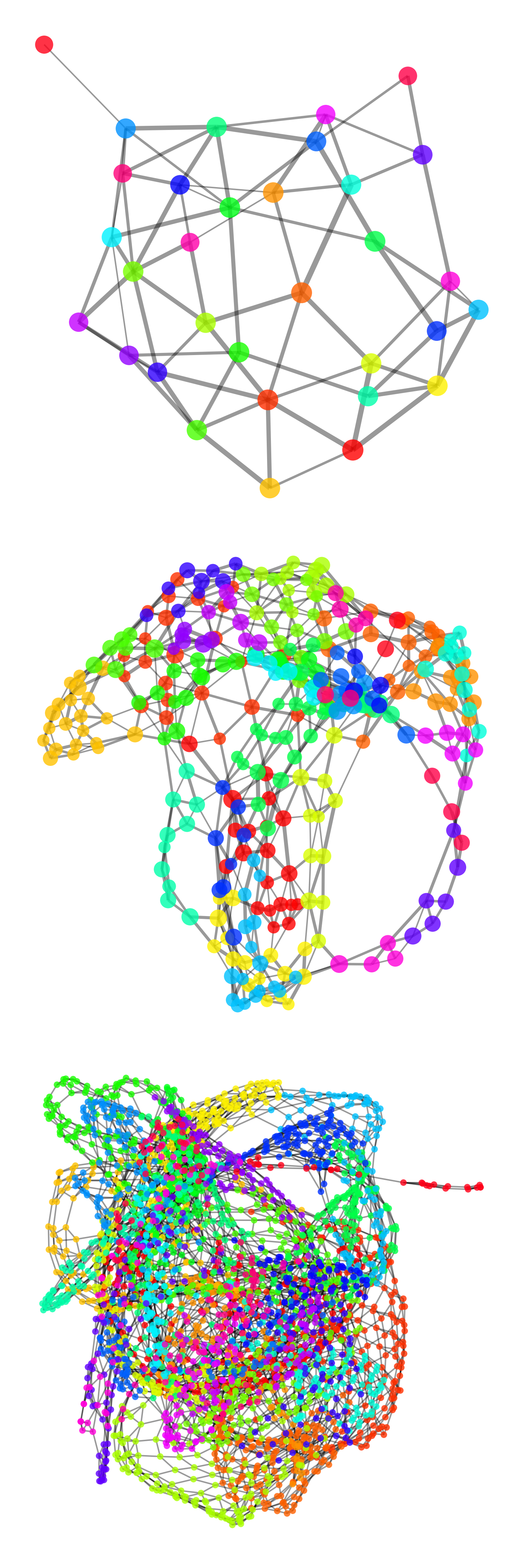}
  \\ \hrule
    \adjincludegraphics[width=.32\textwidth,
		trim={0 {0\height} 0 {0.\height}},clip]{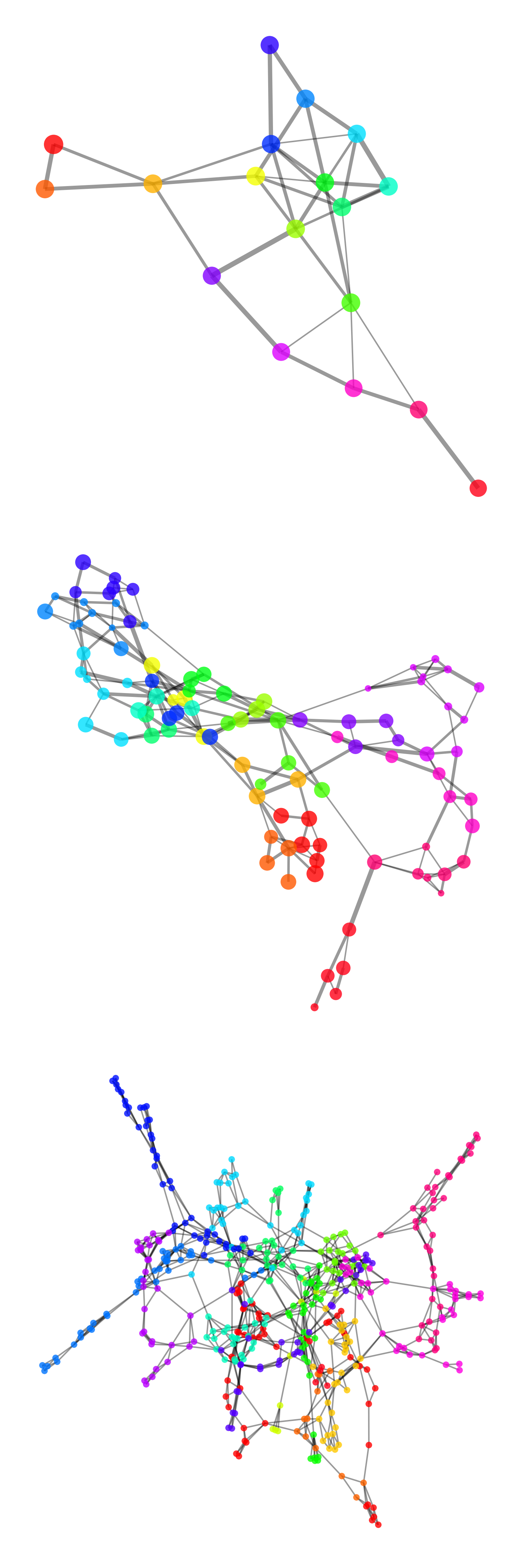}
	\adjincludegraphics[width=.32\textwidth,
		trim={0 {0\height} 0 {0.\height}},clip]{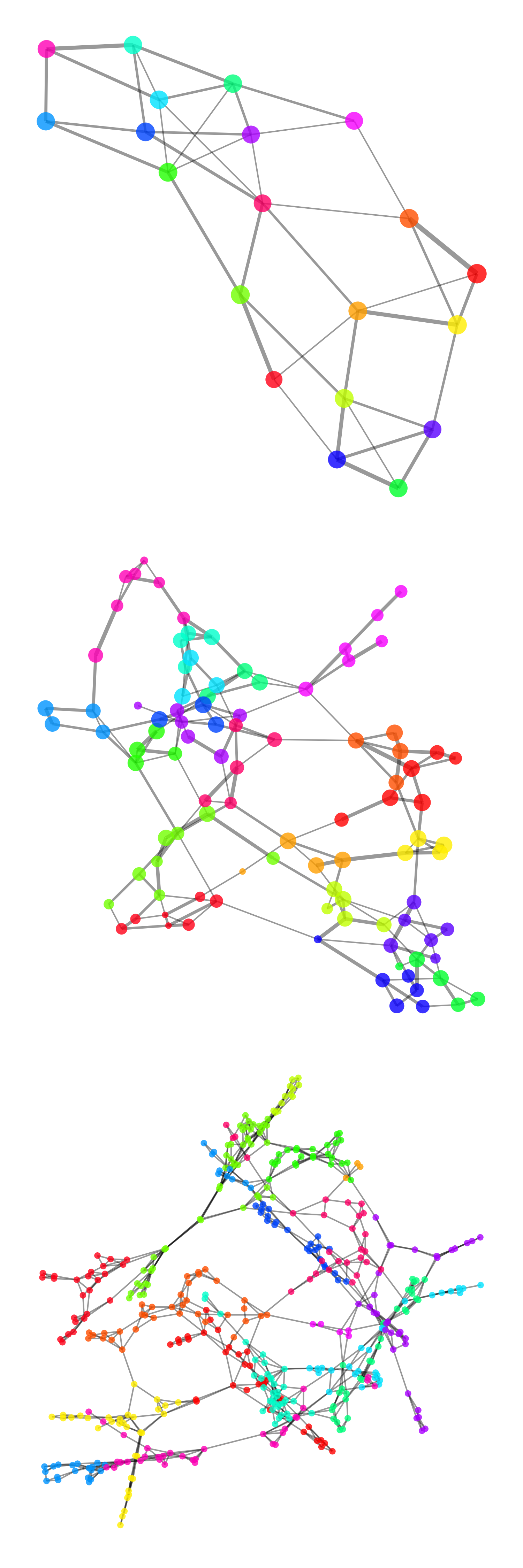}
	\adjincludegraphics[width=.32\textwidth,
		trim={0 {0\height} 0 {.0\height}},clip]{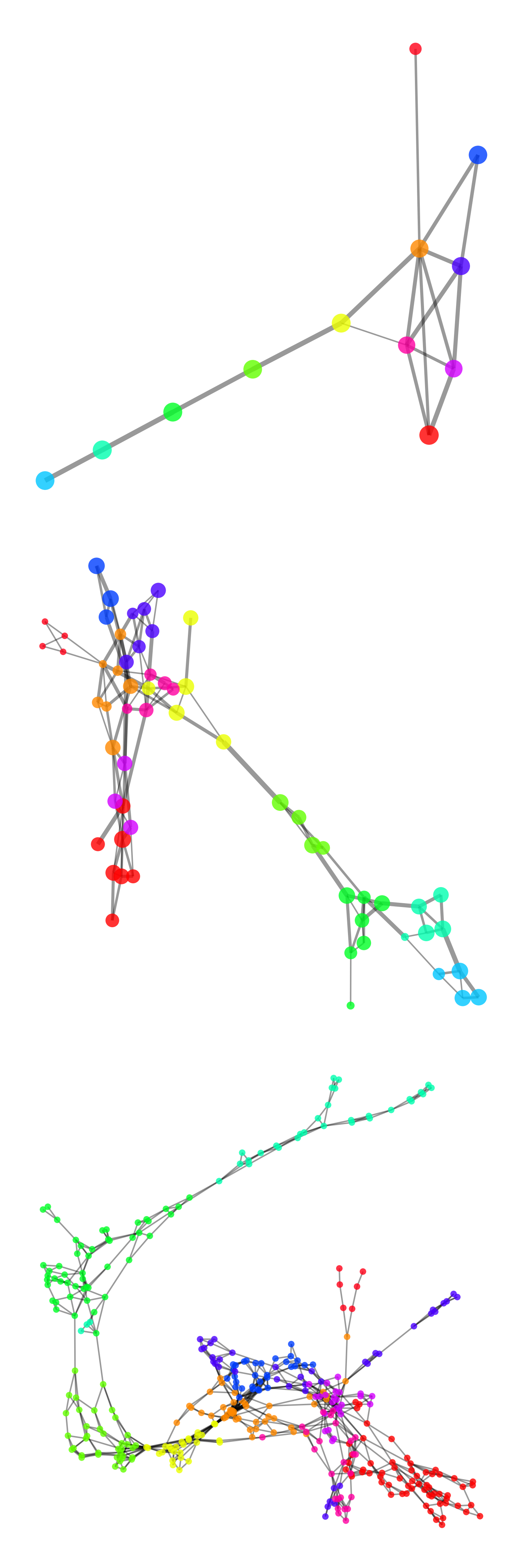}
    \end{minipage} \hfill
    \end{tabular}
\end{center}
\caption{Sample hyper-graphs at 3 levels generated by different models shown at the bottom with training samples at the top. }
\end{figure}

\begin{figure}[h!] \label{fig:sample_HG3}
\begin{center}
    \begin{minipage}{.32\linewidth}
    \begin{center} PPG \end{center}
    \adjincludegraphics[width=.32\textwidth,
		trim={0 {0\height} 0 {0.\height}},clip]{./FigureTable/PPGL_train_1.png}
    \adjincludegraphics[width=.32\textwidth,
		trim={0 {0\height} 0 {0.\height}},clip]{./FigureTable/PPGL_train_2.png}
    \adjincludegraphics[width=.32\textwidth,
		trim={0 {0\height} 0 {0.\height}},clip]{./FigureTable/PPGL_train_3.png} 
  \\ \hrule
    \adjincludegraphics[width=.32\textwidth,
		trim={0 {0\height} 0 {0.\height}},clip]{./FigureTable/PPGL_model_1.png}
	\adjincludegraphics[width=.32\textwidth,
		trim={0 {0\height} 0 {0.\height}},clip]{./FigureTable/PPGL_model_2.png}
	\adjincludegraphics[width=.32\textwidth,
		trim={0 {0\height} 0 {.0\height}},clip]{./FigureTable/PPGL_model_3.png}
    \end{minipage} \hfill
\end{center}
\end{figure}

\section{Experimental details} 
\label{apdx:model_arch}
GRAN enjoyed 7 layers of GNNs with one round of message passing.
Hidden dimensions are set to 128 for [Ego, RCG], 256 for Point Cloud and 512 for [Protein, PPG] for GRAN,
while we used smaller hidden dimensions of 64 for [Ego, RCG, Point Cloud, Protein] and 128 for PPG.

We use the same GNN architecture as GRAN, with 7 layers of GNNs with one round of message passing, but we choose smaller hidden dimensions, setting it to 64 for [Ego, RCG, Point Cloud], and 128 for [Protein, PPG].
For both GRAN and MRG, the number of mixtures is set $K=20$ and block size and stride are both set to 1.
In general, MRG models uses less parameters compared to GRAN. The comparison of total number of parameters of MRG and GRAN are listed in appendix \ref{apdx:model_arch}.
MRG models are training by the Adam optimizer \cite{kingma2014adam} with learning rate of 5e-4.

\begin{table}[h!]
\begin{center}
\caption{ \captionsize
Number of trainable parameters of GRAN vs MRG models.}
\label{table:model_sizes}
	\begin{minipage}{1.\linewidth}
		\centering
    
\begin{tabular}{
l |ccccc}
                    & {\textbf{Protein}} & {\textbf{3D Point Cloud}} & {\textbf{Ego}} & {\textbf{PPG}} & {\textbf{RCG}}    \\ \toprule
\textbf{GRAN}       & 1.75$e^{7}$  & 5.7$e^{6}$ & 1.5$e^{7}$  & 1.77$e^{7}$ & 1.??$e^{7}$ \\
\textbf{MRG-B}         & 7.36$e^{6}$  & 8.09$e^{6}$ & 7.31$e^{6}$  & 1.12$e^{7}$  & 4.17$e^{6}$  \\
\textbf{MRG}         & 9.06$e^{6}$  & 1.20$e^{7}$ & 8.96$e^{6}$  & 1.47$e^{7}$ & 5.94$e^{6}$   \\
\hline
\end{tabular}

	\end{minipage}
\end{center}
\end{table}

.

\end{appendices}

\end{document}